\documentclass[11pt]{article} % For LaTeX2e
\pdfoutput=1
\usepackage[utf8]{inputenc}
\usepackage{amsfonts}
\usepackage{mathtools}
\usepackage{verbatim}
\usepackage{subfiles}

\usepackage{enumitem} 
\usepackage[noend]{algpseudocode}
\usepackage{fancyhdr, lastpage}
\usepackage{natbib,amsmath,amssymb}
\RequirePackage{jmlrutils}
\usepackage{fullpage}

\usepackage[margin=1.2in]{geometry}
\DeclareMathOperator{\BigOm}{\mathcal{O}}
\DeclareMathOperator{\BigOmtil}{\widetilde{\mathcal{O}}}

\newcommand{\BigOh}[1]{\BigOm\left({#1}\right)}
\newcommand{\BigOhTil}[1]{\BigOmtil\left({#1}\right)}

\DeclareMathOperator{\BigWm}{\Omega}
\newcommand{\BigOmega}[1]{\BigWm\left({#1}\right)}

\newcommand{\vol}{\mathrm{vol}}
\newcommand{\adj}{*}

\newcommand{\Gambar}{\overline{\Gamma}}

\newcommand{\Gamsb}{\Gamma_{\mathrm{sb}}}

\newcommand{\sphereGamin}{\calS_{\Gamin}}
\newcommand{\normGamax}{\|\Gamax^{1/2}(\cdot)\|_2}

\newcommand{\Gamin}{\mathtt{\Gamma}_{\min}}
\newcommand{\Gamax}{\mathtt{\Gamma}_{\max}}

\newcommand{\calE}{\mathcal{E}}

\newcommand{\calN}{\mathcal{N}}
\newcommand{\calF}{\mathcal{F}}
\newcommand{\calT}{\mathcal{T}}

\newcommand{\calS}{\mathcal{S}}

\newcommand{\calA}{\mathcal{A}}

\newcommand{\calB}{\mathcal{B}}
\newcommand{\op}{\mathrm{op}}

\newcommand{\calH}{\mathcal{H}}

\newcommand{\itlamax}{\lambda_{+}}

\newcommand{\Gramm}{\Gamma}
\newcommand{\GrammB}{\Gamma^{B_*}}
\newcommand{\sigmaU}{\sigma_u}

\newcommand{\KL}{\mathrm{KL}}

\newcommand{\matU}{\mathbf{U}}
\newcommand{\matV}{\mathbf{V}}
\newcommand{\matSig}{\mathbf{\Sigma}}

\newcommand{\matX}{\mathbf{X}}
\newcommand{\mateps}{\mathbf{E}}
\newcommand{\cond}{\mathrm{cond}}

\newcommand{\OLS}{\mathsf{OLS }}

\newcommand{\N}{\mathbb{N}}
\newcommand{\Ast}{A_{*}}
\newcommand{\Bst}{B_{*}}
\newcommand{\ALS}{\widehat{A}}

\renewcommand{\ast}{a_{*}}
\newcommand{\als}{\widehat{a}}

\newcommand{\Od}{\mathrm{O}(d)}

\newcommand{\Skew}{\mathrm{Skew}}
\newcommand{\eig}{\rho}

\newcommand{\block}{B}
\newcommand{\Noise}{\mathbf{E}}
\newcommand{\noiseb}{W}
\newcommand{\noise}{\eta}
\newcommand{\direc}{w}
\newcommand{\C}{\mathbb{C}}

\newcommand{\R}{\mathbb{R}}
\newcommand{\I}{\mathbb{I}}
\newcommand{\Exp}{\mathbb{E}}

\newcommand{\tr}{\operatorname{tr}}

\newcommand{\iid}{\stackrel{\mathclap{\text{\scriptsize{ \tiny i.i.d.}}}}{\sim}}

\renewcommand{\Pr}{\mathbb{P}}

\newtheorem{thm}{Theorem}[section]

\newtheorem{lem}[thm]{Lemma}
\newtheorem{cor}[thm]{Corollary}

\newtheorem{prop}[thm]{Proposition}

\newtheorem{defn}{Definition}[section]

\newtheorem*{rem}{Remark}

\usepackage{mathtools}
\DeclarePairedDelimiter\ceil{\lceil}{\rceil}
\DeclarePairedDelimiter\floor{\lfloor}{\rfloor}


\algnewcommand\algorithmicinput{\textbf{INPUT:}}
\algnewcommand\INPUT{\item[\algorithmicinput]}
\algnewcommand\algorithmicoutput{\textbf{OUTPUT:}}
\algnewcommand\OUTPUT{\item[\algorithmicoutput]}

\newcommand{\opnormbig}[1]{\left\lVert #1 \right\rVert_{\mathrm{op}}}

\newcommand{\opnorm}[1]{\lVert #1 \rVert_{\mathrm{op}}}
\newcommand{\fronorm}[1]{\lVert #1 \rVert_{\mathrm{F}}}

\newcommand{\maxs}[1]{{\color{magenta}{#1}}}
\newcommand{\horia}[1]{{\color{purple}{#1}}}
\newcommand{\stephen}[1]{{\color{red}{#1}}}

\title{Learning Without Mixing:\\ Towards A Sharp Analysis of Linear System Identification}
\usepackage{times}

\author{ Max Simchowitz\thanks{Department of Electrical Engineering and Computer Science, UC Berkeley, Berkeley CA.} \and
Horia Mania\footnotemark[1]
\and
Stephen Tu\footnotemark[1]
\and
Michael I. Jordan\thanks{Department of Statistics, UC Berkeley, Berkeley CA. }~\footnotemark[1]
\and
Benjamin Recht~\footnotemark[1]}

\begin{document}
\maketitle

\begin{abstract}
%System identification is a fundamental problem in time-series analysis, control
%theory, and reinforcement learning.
%%
%Despite its importance, a sharp non-asymptotic analysis for the number of
%trajectories from an unknown dynamical system needed to identify its parameters
%remains an open question, even in the special case when the dynamics are governed by linear
%equations.
%%
%In this paper, we take an important step towards a non-asymptotic theory for system identification.
%We prove that the ordinary least-squares (OLS) estimator attains nearly minimax
%optimal performance for the identification of linear dynamical systems from
%a single observed trajectory.
%%
%Our analysis relies on a generalization of Mendelson's small-ball method to dependent data,
%eschewing the use of standard mixing-time arguments.
%%
%We capture the correct
%signal-to-noise behavior of the problem, showing that \emph{more unstable} linear
%systems are \emph{easier} to estimate.
%%
%This behavior is qualitatively different from arguments which rely on mixing-time
%calculations that suggest that unstable systems are more difficult to estimate.
%%
%Finally, our proof techniques generalize to a class of linear response
%time-series.

We prove that the ordinary least-squares (OLS) estimator attains nearly minimax
optimal performance for the identification of linear dynamical systems from
a single observed trajectory.
Our upper bound relies on a generalization of Mendelson's small-ball method to dependent data,
eschewing the use of standard mixing-time arguments.
Our lower bounds reveal that these upper bounds match up to logarithmic factors.
In particular, we capture the correct
signal-to-noise behavior of the problem, showing that \emph{more unstable} linear
systems are \emph{easier} to estimate.
This behavior is qualitatively different from arguments which rely on mixing-time
calculations that suggest that unstable systems are more difficult to estimate.
We generalize our technique to provide bounds for a more general class of linear response
time-series.

\end{abstract}

% \begin{keywords}
% Linear dynamical systems, autoregressive processes, time series, system identification, empirical process theory
% \end{keywords}

%!TEX root = LWM.tex
\section{Introduction}

% As robotics become commonplace in everyday life, it is crucial that such hardware satisfy robust performance and safety guarantees. Such guarantees are underwritten by confident \emph{system identification} of the underlying robot's dynamics. 
%
%%
%%
%In this paper, we take an important step towards a non-asymptotic theory for system identification.
%
%It is not-yet understood how many number of measurements are required to identify an unknown system, even for systems whose dynamics take simple parametric forms. 

System identification---the problem of estimating the parameters of a dynamical system given a time series of its trajectories--- is a fundamental problem in time-series analysis, control theory, robotics, and reinforcement learning.
Despite its importance, sharp, non-asymptotic analyses for the sample complexity of system identification are rare. In particular, it is not known how many trajectories required to identify the parameters of an unknown \emph{linear} system.
Properly characterizing this sample complexity would have profound implications, since accurate error bounds are indispensable for designing robust and high-performing control systems. It is important that the bounds be sharp, in the sense that they do not drastically \emph{overestimate} the number of required measurements from system trajectories, which are often time-consuming and prohibitively expensive to collect. 
More broadly, a deeper understand of system identification would inform other statistical problems where one wishes to learn from non-i.i.d.\ or time-correlated data.

We focus on the problem of identifying a discrete-time \emph{linear dynamical system}
from an observed trajectory. Such systems are described by two parameter matrices 
$\Ast$ and $B_*$, and the dynamics evolve according to the law $X_{t+1} = \Ast X_t + B_* u_t + \noise_t$, where $X_t\in \R^d$ is the state of the system, $u_t$ is the input of the system, and $\noise_t \in \R^d$ denotes unobserved process noise. Linear systems are fundamental in control theory, since they are able to capture the behavior of many natural systems and also able to accurately describe the evolution of an even broader class of systems near their equilibria.
%
% with many applications relying on optimal LQR or iLQR control. 
Despite the importance of understanding the statistical properties of system identification, the relationship between the matrix $\Ast$ and the statistical rate for estimating this matrix remains poorly understood. 
We note that the larger the state vectors $X_t$ are in comparison to the process noise, the larger the \emph{signal-to-noise ratio} 
for estimating $\Ast$ is. As a result, larger matrices $\Ast$ (larger in an appropriate sense, discussed later) lead to states $X_t$ of larger norm, which in turn should make the estimation of $\Ast$ easier. However, it is difficult to theoretically formalize this intuition because the sequence of measurements $X_0, X_1, \ldots, X_{T - 1}$ used for estimation is not i.i.d. and it is dependent on the noise $\noise_0, \noise_1, \ldots, \noise_{T - 2}$. Even the computationally 
straightforward ordinary least-squares ($\OLS$) estimator is difficult to analyze. Standard analyses for $\OLS$ on random design linear regression~\citep{hsu14} cannot be used due to the dependency between the covariates $X_t$ and the process noise $\noise_t$.

In the statistics and machine learning literature, correlated data is usually
dealt with using mixing-time arguments~\citep{yu94}, which relies
on fast convergence to a stationary distribution that
allows correlated
samples to be treated roughly as if they were independent.  While this approach has
been successfully used to develop generalization bounds for time-series
data~\citep{mohri07}, a fundamental limitation of mixing-time arguments is that
the bounds deteriorate when the underlying process is slower to mix. 
In the case of linear systems, this behavior is qualitatively incorrect.
For linear systems, the rate of mixing is intimately tied to 
the eigenvalues of the matrix $\Ast$, specifically the \emph{spectral radius} $\rho(\Ast)$.
When $\rho(\Ast) < 1$ (i.e. when the system is \emph{stable}), 
the process mixes to a stationary distribution at a rate that 
deteriorates as $\rho(\Ast)$ approaches the boundary of one.
%On the other hand, when $\rho(\Ast) \geq 1$, the system does not mix at all.
%
However, as discussed above, as $\rho(\Ast)$ increases
we expect estimation to become easier 
due to better signal-to-noise ratio, and not harder as mixing-time arguments suggest. 
%
%In the case of linear systems, this behavior is qualitatively incorrect.
%The eigenvalues of the matrix $\Ast$ determine the mixing-time of the 
%system's trajectories; larger eigenvalues yield slower mixing rates
%to the stationary distribution, but states of larger norm.  
%As discussed above, in this situation estimation becomes easier due 
%to better signal-to-noise ratio, and not harder as mixing-time arguments suggest. 
%
We note that recent work by \cite{faradonbeh17a} studying the estimation
problem for linear systems relies in the stable case
on concentration of measure arguments which also degrade as the mixing-time of the
system grows.
%and more recently
%to derive bounds for OLS in our setting~\citep{faradonbeh17a} \horia{(does this paper use mixing arguments?)} \stephen{(not directly, but they do use the spirit of these kind of arguments by arguing about convergence to the stationary covariance matrix)}, 
 
We address these difficulties and offer a new statistical analysis of the ordinary least-squares ($\OLS$) estimator  
of the dynamics $X_{t + 1} = \Ast X_{t} + \noise_t$ with no inputs, when the spectral radius of $\Ast$ is at most one ($\rho(\Ast) \leq 1$, a regime known 
as \emph{marginal stability}). Our results, detailed in Section~\ref{sec:results}, show 
that the statistical performance of $\OLS$ is determined by the minimum eigenvalue of the (finite-time) controllability Gramian $\Gramm_T = \sum_{s = 0}^{T - 1}\Ast^s(\Ast^\top)^s$. The controllability Gramian is a fundamental quantity in the theory of linear systems; the eigenvalues of the Gramian quantify how much white process noise $\noise_t \overset{i.i.d}{\sim} \mathcal{N}(0,\sigma^2 I)$ can excite the system. We show that a larger $\lambda_{\min}(\Gramm_T)$ leads to faster estimation of $\Ast$ in operator norm, and we also prove that up to log factors the $\OLS$ estimator is minimax optimal. Furthermore, in Section~\ref{sec:meta_thm} we offer similar statistical guarantees for a more general class of linear response time-series.

%!TEX root = LWM.tex
\subsection{Related Work}

Most directly related to our work is a recent series of papers by 
\cite{faradonbeh17a,faradonbeh17b}, who study the
linear system identification problem by proving a non-asymptotic rate
on the convergence of the OLS estimator to the true system matrices.
%Faradonbeh et al.\ distinguish between two cases: (a) when the $\Ast$ matrix
%is stable (i.e. spectral radius $\rho(\Ast)$ is bounded by one), and
%(b) when it is not.
In the regime where $\Ast$ is stable, Faradonbeh et al.\ recover 
a similar rate as our result. The major
difference is that the dependence of their analysis on the spectral properties of
$\Ast$ are qualitatively suboptimal, and difficult to interpret precisely.
Their analysis is based on separately establishing concentration of the
sample covariance matrix $\sum_{t=1}^{T} X_t X_t^\top$ to the stationary
covariance matrix and bounding the martingale difference term $\sum_{t=1}^{T} X_t \noise_t$.
This decoupled analysis inevitably picks up a dependence on the condition
number of the stationary covariance matrix, which means that
as the system becomes more unstable, their bound deteriorates. 
Indeed, such an strategy is unable to provide any insight into the behavior when,
for example, $\Ast$ is a scaled orthogonal matrix.
%\maxs{What does their bound say about marginally stable? Does it even hold} 
On the other hand, our analysis does not decouple the two terms, and as
a result our bounds only degrade in the \emph{logarithm} of the condition
number of the finite-time controllability Gramian $\Gamma_T$.
\cite{faradonbeh17a} also provide a bound in the \emph{unstable regime}, which
we believe can be sharpened using our analysis techniques which couple the
covariate- and noise-processes. We leave this to future work. Moreover,
our analysis of one-dimensional, unstable systems corroborates the linear
convergence behavior that \cite{faradonbeh17a} obtain for ``explosive'' systems,
which are systems where
\emph{all} eigenvalues of $\Ast$ lie outside the complex unit disk.

Another closely related work is the scalar analysis by \cite{rantzer18}.
In fact, our proof technique for scalar systems can be seen as an extension of
his technique. The main difference is that by more carefully tracking the terms
that appear in the moment generating function of the noise and covariate processes, we are able to discriminate behaviors that
arise when $\Ast$ is stable versus unstable, and uncover a linear rate of convergence
in the unstable regime.

Our result qualitatively matches the behavior of the
rate given in \cite{dean17},
in that the key spectral quantity governing the rate of convergence is the
minimum eigenvalue of the finite-time controllability Gramian.
The major difference is that the analysis
in Dean et al.\ uses multiple independent trajectories, and discards all but the last
state-transition in each trajectory. This decouples the covariates, and reduces
the analysis to that of random design linear regression with independent covariates.
We note, however, that the analysis in Dean et al.\ applies even when $\Ast$ is
unstable.

More broadly, there has been recent interest in non-asymptotic analysis of linear system
identification problems. Some of the earlier non-asymptotic literature in system identification
include \cite{campi2002finite} and \cite{vidyasagar2008learning}.
The results provided in this line of work are often quite conservative,
featuring quantities which are exponential in the degree of the system.
Furthermore, the rates given are often difficult to interpret.
More recently, \cite{shah12} pose the problem of recovering
a single-input, single-output (SISO) LTI system from linear measurements in the frequency domain as a sparse recovery
problem, proving polynomial sample complexity for recovery in the $\calH_2$-norm.
\cite{hardt16} show that under fairly restrictive assumptions on the
$\Ast$ matrix, projected gradient descent recovers the state-space representation 
of an LTI system with only a polynomial number of samples. 
The analysis from both Shah et al. and Hardt et al. both degrade polynomially
in $\frac{1}{1-\rho(\Ast)}$, where $\rho(\Ast)$ is the spectral radius of underlying $\Ast$.
On the other hand, \cite{hazan17} propose a new spectral filtering algorithm 
for online prediction of linear systems where the rates do not degenerate as $\rho(\Ast) \to 1$,
with the caveat that the analysis only applies to symmetric $\Ast$ matrices. \cite{hazan18} extends the analysis to diagonalizable matrices, but the obtained error rates are polynomial in problem parameters. Both works also consider the more general setting where $X_t$ is observed indirectly via $Y_t = CX_t$ for an unknown observation matrix $C$.
Moreover, the main metric of interest in both \cite{hardt16} and \cite{hazan17,hazan18}
is the prediction error.  It is not clear how
prediction error guarantees can be used in downstream robust control synthesis
applications, whereas the operator norm bounds we provide can be used as direct
inputs into robust synthesis for optimal control problems~\citep{dean17}.

The most well-established technique in the statistics literature for dealing
with non-independent, time-series data is the use of mixing-time arguments~\citep{yu94}.
In the machine learning literature, mixing arguments have been used to develop
generalization bounds~\citep{mohri07,mohri08,kuznetsov17,mcdonald17b} which
are analogous to the classical generalization bounds for i.i.d.\ data.
As mentioned previously, a fundamental limitation of mixing-time arguments is that
the bounds all degrade as the mixing-time increases. This has two implications for
linear system identification: (a) none of these existing results can correctly capture
the qualitative behavior as the $\Ast$ matrix reaches instability, and (b)
these techniques cannot be applied to the regime where $\Ast$ is unstable, for which
estimation is not only well-posed, but should be quite easy.
It is for these reasons we do not pursue such arguments in this work.

%\maxs{what about Cyril's paper?}
%\stephen{TODO: maybe make a short blurb about Mendelson's small ball method.}

%\stephen{---------------------------------------------------------}
%
%\paragraph{Estimation papers.}
%
%\stephen{michigan paper~\cite{faradonbeh17a,faradonbeh17b}}
%\\
%\stephen{dean et al~\cite{dean17}}
%\\
%\stephen{cyril zhang paper~\cite{hazan17}}
%\\
%\stephen{tengyu paper~\cite{hardt16}}
%\\
%\stephen{rantzer paper}
%\\
%\stephen{campi papers}
%
%\paragraph{Mixing papers.}
%
%\stephen{yu94~\cite{yu94}, mohri~\cite{mohri07,mohri08}, kuznetsov, mcdonald and shalizi~\cite{mcdonald17b}, my recent LQR+LSTD paper}

%!TEX root = LWM.tex
\section{Results}
\label{sec:results}

    In this work, we consider both the specific problem of estimating linear dynamical systems, and a more general problem of linear estimation in time series.
    In both cases we measure the estimation error in the operator norm.
    In the case of linear dynamical systems we analyze the statistical performance of the $\OLS$ estimator for the parameter $\Ast$
    from a single observed trajectory $X_1, \ldots, X_{T+1}$ satisfying $X_{t + 1} = \Ast X_t + \noise_t$, where $X_0 = 0$ and $\noise_t \iid \calN(0, \sigma^2 I_d)$:
    \begin{align}
    \ALS(T) &:= \arg\min_{A \in \R^{d \times d}} \sum_{t=1}^T \frac{1}{2}\|X_{t + 1} - A X_t\|_2^2 \:.
    \end{align}

In Section~\ref{sec:linear_systems} we present upper bounds on
$\opnorm{ \ALS - \Ast }$ which hold for any $\Ast$ with $\rho(\Ast) \leq 1$.
In Section~\ref{sec:lower_bounds}, we show that these upper bounds are
nearly optimal in many regimes of interest.
Finally, Section~\ref{sec:meta_thm} states a general result, Theorem~\ref{main_thm},
which applies to arbitrary covariate processes with linear responses. 

\textbf{Notation:} We let $\|\cdot\|_{\op}$ denote the operator norm of a matrix, $\calS^{d-1}$ denote the unit sphere in $\R^d$. Given a symmetric matrix $A \in \R^{d\times d}$, we let $\lambda_{\max}$ and $\lambda_{\min}$ denote the largest, and smallest eigenvalue of $A$.

\subsection{Linear Dynamical Systems}
\label{sec:linear_systems}
Here we consider systems $X_{t + 1} = \Ast X_t + \noise_t$, where $\noise_t \sim \mathcal{N}(0,\sigma^2 I_d)$ and $X_0 = 0$. Our bounds are stated in terms of the finite-time controllability Gramian of the system $\Gramm_t := \sum_{s=0}^{t-1} (\Ast^{s})(\Ast^s)^{\top}$, which captures the magnitude of the excitations induced by the process noise. Indeed, we can write $X_t$ explicitly as
\begin{eqnarray}\label{eq:recursion_eq}
X_{t} = \sum_{s=1}^{t} \Ast^{t-s}\noise_{s-1}~~\text{ which implies that}~~\Exp[X_tX_t^{\top}] = \sigma^2\Gramm_t~.
\end{eqnarray}
Hence, the expected covariance can be expressed in terms of the Gramians via $\Exp[\sum_{t=1}^T X_tX_t^\top] = \sigma^2\cdot \sum_{t=1}^T \Gramm_t$. As is standard in analyses of least-squares, ``larger'' covariates/covariance matrices correspond to faster rates of learning.
% For example,, and hence, $\Exp\|X_t\|_2^2 = \tr\left(\Gramm_t\right)$. Furthermore, the identity implies th
We are ready to state our first result, proved in Section~\ref{sec:pf_stable_thm}:

\begin{thm}\label{stable_thm} Fix $\delta \in (0,1/2)$ and consider the linear dynamical system $X_{t+1} = \Ast X_t + \noise_t $, where $\Ast$ is a marginally stable matrix in $\R^{d\times d}$ (i.e. $\rho(\Ast) \leq 1$), $X_0 = 0$, and $\noise_t  \iid \calN(0,\sigma^2 I)$. Then there exist universal constants $c,C > 0$ such that
\begin{align}
\Pr\left[\opnormbig{\ALS(T)-\Ast} >\frac{C}{\sqrt{T\lambda_{\min}\left(\Gramm_{k}\right)}} \sqrt{d\log\frac{d}{\delta} + \log \det (\Gramm_T \Gramm_k^{-1})} \right] \le  \delta,
\end{align}
for any $k$ such that $\frac{T}{k} \ge c(d\log(d/\delta)+\log \det (\Gramm_T \Gramm_{k}^{-1}))$ holds.
\end{thm}
Note that $\sigma^2$ does not appear in the bound from Theorem~\ref{stable_thm} because scaling the noise also rescales the covariates. In Appendix~\ref{sec:appendix:kst}, we show that for any marginally stable $\Ast$,
we can always choose a $k \geq 1$ provided $T$ is sufficiently large. Therefore, even when $\rho(\Ast) = 1$ and the system does not mix, we obtain finite-sample estimation guarantees which also guarantees consistency of estimation. In many cases, these rates are qualitatively no-worse than random-design linear regression with independent covariates (Theorem~\ref{cor:consistent} and Remark~\ref{rem:consistent}).

In general, $\lambda_{\min}\left(\Gramm_k\right)$ is a nondecreasing function of the block length $k$. The intuition for this is that larger $k$ takes into account more long-term excitations to lower bound the size of our covariance matrix. However, as we use longer blocks, our high probability bounds degrade. Thus, the optimal block length is the maximal value $k$ which satisfies Theorem~\ref{stable_thm}.

The dependence on the minimum eigenvalue of the Gramian $\lambda_{\min}\left(\Gramm_k\right)$ has two interpretations. From the \emph{statistical} perspective, we have that $\frac{1}{2k\cdot \sigma^2}\Exp[\sum_{t=1}^{2k} X_tX_t^\top] = \frac{1}{2k}\sum_{t= 1}^{2k}\Gramm_t \succeq \frac{1}{2}\Gramm_{k} \succeq \frac{1}{2}\lambda_{\min}\left(\Gramm_k\right) \cdot I $. Thus, $\lambda_{\min}\left(\Gramm_k\right)$ gives a lower bound on the smallest singular value of the covariance matrix associated with the first $2k$ covariates. In fact, one can also show (see~\eqref{eq:paley_zygmund}) that for any $t_0 \ge 0$, we still have $\frac{1}{2k\cdot \sigma^2}\Exp[\sum_{t=t_0 + 1}^{t_0 + 2k} X_tX_t^\top | X_{t_0}]~\succeq~\frac{1}{2}\Gramm_{k}$, so that $\lambda_{\min}\left(\Gramm_k\right)$ in fact lower bounds the covariance of \emph{any} subsequence of $2k$-covariates. Theorem~\ref{stable_thm} thus states that the larger the expected covariance matrix, the faster $\Ast$ is estimated. Note that $\Gamma_k \succeq I$ for all $k \ge 1$.

The second interpretation is \emph{dynamical}. The term $\lambda_{\min}\left(\Gramm_k\right)$ corresponds to the ``excitability'' of the system, which is the extent to which the process noise $\noise_t$ influences future covariates. This can be seen from~\eqref{eq:recursion_eq}, where the slower $(\Ast^{t_0})(\Ast^{t_0})^\top$ decays as $t_0$ grows, the larger the contribution of process-noise from $t_0$ steps before, $\noise_{t - t_0-1}$. This is precisely the reason why linear systems with larger spectral
radii mix slowly, and do not mix when $\rho(\Ast) \ge 1$.
In this light, Theorem~\ref{stable_thm} shows that with high-probability, the more a linear system is excited by the noise $\noise_t$, the easier it is to estimate the parameter matrix $\Ast$. For stable systems with $\rho(\Ast) < 1$, the following corollary removes the explicit dependence on the block length $k$ for large values of $T$:
\begin{cor}\label{cor:mixed_cor}
Suppose that $\rho(\Ast) < 1$. Then the limit $\Gamma_{\infty} := \lim_{t \to \infty} \Gamma_t$ exists, and there is a time $T_0$ depending on $\Ast$ and $\delta$ such that the following holds w.p. $1-\delta$ for all $T > T_0$:
\begin{eqnarray}
\opnormbig{\ALS(T)-\Ast} \le \BigOh{\sqrt{ \frac{d \cdot \log\left(\frac{d}{\delta} \right)}  {T\lambda_{\min}\left(\Gramm_{\infty}\right)}}}  \,.
\end{eqnarray}
\end{cor}
The above corollary uses the fact that if $\rho(\Ast) < 1$, one can bound $\|\Ast^k\|_{\op} \le \mathrm{poly}(k) \rho(\Ast)^k$; where the polynomial $\mathrm{poly}(k)$ is related to the $\mathcal{H}_{\infty}$-norm of the linear system, a core concept in control theory. For an extended discussion on this relationship, we direct the reader to~\cite{tu2017non}. Corollaries~\ref{cor:consistent} and~\ref{cor:diag} in the appendix gives an analogue of Corollary~\ref{cor:mixed_cor} which holds even if $\rho(\Ast) = 1$. We now explicitly describe the consequences of Theorem~\ref{stable_thm} for three illustrative classes of linear systems:

%Even a when $\rho(A^*) = 1$, we can find a positive block length $k$ which satisfies the hypothesis of Theorem~\ref{stable_thm}, provided $T$ is sufficiently large.

\begin{enumerate}
\item \textbf{Scalar linear system.} In this case the states $X_t$ and the parameter $\Ast$ are scalars, and denoted $a_* = \Ast$. For $|a_*| \le 1$, we can apply Theorem~\ref{stable_thm} with block length $k = \BigOm(T/\log(1/\delta))$. This then guarantees that $|\widehat{a} - a_*| \leq \BigOm\left(\sqrt{\log(1/\delta)/\left(T \sum_{t = 1}^{k_*} a_*^{2t} \right)}\right)$ with probability $1 - \delta$. In Appendix~\ref{sec:1d_appendix}, we show this statistical rate is minimax optimal (Theorem~\ref{thm:info_lb_1d}). Moreover, we offer a specialized analysis for the scalar case (Theorem~\ref{thm:one_d_thm}) which yields sharper constants and also applies to the unstable case $|a_*| > 1$, matching the lower bounds of Theorem~\ref{thm:info_lb_1d}. Stated succinctly, our results in Appendix~\ref{sec:1d_appendix} imply that the $\OLS$ estimator satisfies with probability $1 - \delta$ error guarantees which can be categorized into three regimes:
\begin{align*}
    |\widehat{a} - a_*| = \begin{cases}
        \Theta\left(\sqrt{\frac{\log(1/\delta)(1 - |a_*|)}{T}}\right)\; &\text{ if }\; |a_*| \leq 1 - \frac{c\log(1/\delta)}{T},\\
        \Theta\left(\frac{\log(1/\delta)}{T}\right)\; &\text{ if }\; 1 - \frac{c\log(1/\delta)}{T} < |a_*| \leq 1 + \frac{1}{T}\\
        \Theta\left(\frac{\log\left(1/\delta\right)}{|a_*|^T}\right) \; &\text{ if }\; 1 + \frac{1}{T} \le |a_*|.
    \end{cases}
\end{align*}
\item \textbf{Scaled orthogonal systems.} Let us assume $\Ast =  \rho \cdot O$ for an orthogonal $d\times d$ matrix $O$ and $|\rho| \leq 1$. In this case, one can verify that $\Gamma_t = I \cdot \sum_{s=0}^{t-1} \rho^{2s}$. A bit of algebra reveals that we can choose the block length $k = \BigOm\left(\tfrac{T}{d\log(d/\delta)}\right)$. Therefore, Theorem~\ref{stable_thm} guarantees that with probability $1 - \delta$:
\begin{align}\label{eq:rho_eye_upper}
\opnorm{\ALS - \Ast} \leq \begin{cases}
    \BigOm\left(\sqrt{(1-|\rho|) \cdot \frac{d\log(d/\delta)}{T}}\right)\; &\text{ if }\; |\rho| \leq 1 - \frac{cd\log(d/\delta)}{T},\\
    \BigOm\left(\frac{d\log(d/\delta)}{T}\right)\; &\text{ if }\; 1 - \frac{cd\log(d/\delta)}{T} < |\rho|.
\end{cases}
\end{align}
When $|\rho| \le 1 - \frac{cd\log(d/\delta)}{T}$, one can sharpen the logarithmic factors (Remark~\ref{rem:improved_mixing}).

	\item \textbf{Diagonalizable linear systems.} We consider a diagonalizable linear system with $\Ast = SDS^{-1}$. We denote by $\rho$ the spectral radius of $\Ast$ and by $\underline{\rho}$ the smallest magnitude of an eigenvalue of $\Ast$. In Appendix~\ref{app:consistent}, we show that we can choose $k$ such that
\begin{align*}
k \geq \frac{T}{c d \log\left(\frac{d\cond(S)}{\delta}\right)}
\end{align*}
This choice of $k$ shows that the $\OLS$ estimator satisfies (Corollary~\ref{cor:diag})
\begin{align*}
\Pr\left[\opnorm{\ALS - \Ast} \leq \BigOm\left(\sqrt{\frac{d\log(d\cond(S)/\delta)}{T\left(1 + \cond(S)^{-2}\sum_{s = 0}^{k-1}\underline{\rho}^{2s}\right)}}\right)\right] \ge 1- \delta
\end{align*}
which could once again be split into a slow and fast rate, as in the examples presented above, depending on the size $\underline{\rho}$ of the least excitable mode of the system defined by $\Ast$. Note that up to a factor of $\log(d\cond(S)/\delta)$, the above bound is no worse than the minimax rate for standard random-design least-squares.
\end{enumerate}
\begin{remark}[Slightly Improved Rates for Stable Systems]\label{rem:improved_mixing} The dependence on $\tr(\Gamma_T)$ comes from naively bounding $\opnorm{\sum_{t=1}^TX_tX_t^\top}$ by $\tr(\sum_{t=1}^TX_tX_t^\top)$ and applying Markov's inequality. For many systems, including strictly stable systems (e.g. Corollary~\ref{cor:mixed_cor}), one can show via Hanson-Wright~\citep{rudelson13} that $\opnorm{\sum_{t=1}^TX_tX_t^\top}$ concentrates below $\BigOh{T\lambda_{\max}(\Gamma_T)} \le \BigOh{T\lambda_{\max}(\Gamma_{\infty})}$. This can be used to replace the dependence on $\tr(\Gamma_T)$ in Theorem~\ref{stable_thm} with~$\BigOh{\lambda_{\max}(\Gamma_T)}$ (and $\tr(\Gamma_{\infty})$ with $\lambda_{\max}(\Gamma_{\infty})$) which will typically remove a factor of $\log d$ in the error rates.
\end{remark}
\begin{remark}[Noise dependence] The estimation guarantee provided by Theorem~\ref{stable_thm} does not depend on the variance $\sigma^2$ of the noise $\noise_t$. This surprising property holds because the size of the variance $\sigma^2$ directly influences the size of the states $X_t$ leading to a cancellation in the signal-to-noise ratio. For Gaussian noise with a general identity covariance $\noise_t \sim \calN(0,\Sigma)$, one can rederive rates from our more general Theorem~\ref{main_thm} to get a more precise dependence on $\Gramm_t$ and $\Sigma$.
Note that if the covariance is known, an alternative estimator would be to choose $\ALS$ to minimize a loss which takes $\Sigma$ into account in the same way that one would for non-dynamic linear regression with heteroskedastic noise, e.g. $\widehat{A}^{\Sigma}(T) := \arg\min_{A \in \R^{d \times d}} \sum_{t=1}^T \frac{1}{2}\left\|\Sigma^{-1/2}\left(X_{t + 1} - A X_t\right)\right\|_2^2$.
\end{remark}
\begin{remark}[Learning with input sequences] We can also consider the case where the system is driven by a known sequence of inputs $u_0,u_1,\dots$ and where $\Bst$ is known. Defining the control Gramian $\GrammB_t := \sum_{s=1}^t \Ast^{t-s}\Bst\Bst^\top \Ast^{t-s}~$, the proof of Theorem~\ref{stable_thm} can be modified to show that, if the inputs are white noise $u_t \overset{i.i.d}{\sim} \mathcal{N}(0,\sigmaU^2 I)$, then there exist universal constants $c,C > 0$ such that, with probability $1-\delta$,
\begin{align*}
\opnorm{\ALS(T)-\Ast} \le \frac{C\sigma^2}{\sqrt{T\lambda_{\min}\left(\sigma^2\Gramm_k + \sigmaU^2 \GrammB_k)\right)}}\sqrt{d\log\left(\frac{1}{\delta} \frac{\tr \left(\sigma^2\Gramm_T + \sigmaU^2 \GrammB_T)\right)}{\lambda_{\min}\left(\sigma^2\Gramm_k + \sigmaU^2 \GrammB_k\right)}\right)}
\end{align*}
for any $k$ such that $\frac{T}{k} \ge cd\log(\frac{\tr(\sigma^2\Gramm_T + \sigmaU^2 \GrammB_T)}{\delta\lambda_{\min}(\sigma^2\Gramm_k + \sigmaU^2 \GrammB_k)})$.

Non-white noise with covariance not equal to a multiple of the identity can be absorbed into $\Bst$. Moreover, other non-Gaussian control input processes $u_t$ could just as easily be accommodated by our Theorem~\ref{main_thm}.
When $\Bst$ is unknown, Theorem~\ref{main_thm} still implies that we can learn $(\Ast,\Bst)$; however, the guarantees are in terms of the operator norm of the concatenation of the errors, $\opnorm{(\widehat{A} - \Ast, \widehat{B} - \Bst)}$; in particular, the bound does not differentiate between the error of $\widehat{A}$ and the error of $\widehat{B}$. We believe that developing guarantees that delineate between the errors in $\widehat{A}$ and in $\widehat{B}$ is an exciting direction for future work.
\end{remark}

\subsection{Lower Bounds for Linear System Identification}
\label{sec:lower_bounds}

    We have seen in Theorem~\ref{stable_thm} and in the subsequent examples that the estimation of linear dynamical systems is easier for systems
    which are easily excitable. It is natural to ask what is the best possible estimation rate one can hope to achieve. To make explicit the dependence of the lower bounds on the spectrum of $\Gamma_t$, we consider the minimax rate of estimation over the set $\eig \cdot \Od$, where $\eig \in \R$ and $\Od$ denotes the orthogonal group. In this case, we can define an \emph{scalar} Gramian $\gamma_t(\eig) := \sum_{s=0}^{t-1} |\eig|^{2s}$, and so that $\Gamma_t := \gamma_t(\eig) \cdot I$. We now show that the estimation rate of the $\OLS$ provided in Theorem~\ref{stable_thm} is optimal up to log factors for $|\rho| \le 1 - \BigOhTil{d/T}$:
    \begin{thm}
    \label{thm:info_lb_d} Fix a $d \ge 2$, $\eig \in \R$, $\delta \in (0,1/4)$, and $\epsilon \le \frac{\eig}{2048}$. Then, there exists a universal constant $c_0$ such for any estimator $\widehat{A}$,
    \begin{eqnarray*}
    \sup_{O \in \Od}\Pr_{\eig O}\left[\left\|\widehat{A}(T) - \eig O\right\|_{\op} \ge \epsilon\right] \ge \delta\; \text{ for any } T \text{ such that }\; T \gamma_{T}(\eig) \leq  \frac{c_0 \left( d + \log\left(1/\delta\right)\right)}{\epsilon^2},
    \end{eqnarray*}
    where $\Od$ is the orthogonal group of $d \times d$ real matrices.
    \end{thm}
    This theorem is proved in Appendix~\ref{sec:lb_proofs}. We can interpret the result
    by considering the following regimes:
\begin{align*}
    \opnorm{\widehat{A} - \Ast} \ge \begin{cases}
        \Omega\left(\sqrt{\frac{(d+\log(1/\delta))\cdot(1 - |p|)}{T}}\right)\; &\text{ if }\; |a_*| \leq 1 - \frac{1}{T},\\
        \Omega\left(\frac{\sqrt{d+\log(1/\delta) }}{T}\right)\; &\text{ if }\; 1 - \frac{1}{T} < |\rho| < 1 + \frac{1}{T}\\
        \Omega\left(\sqrt{\frac{d+\log(1/\delta)}{T|\rho|^T}}\right) \; &\text{ if }\; 1 + \frac{1}{T} \le |\rho|.
    \end{cases}
\end{align*}
Comparing to~\eqref{eq:rho_eye_upper}, we see that for $|\rho| \le 1 - \BigOhTil{d/T}$, our upper and lower bounds coincide up to logarithmic factors. In the regime $\rho \in [1 - \BigOhTil{d/T},1]$, our upper and lower bounds differ by a factor of $\BigOhTil{\sqrt{d + \log(1/\delta)}}$. We conjecture that our upper bounds qualitatively describe the performance of $\OLS$. That is, we conjecture that the least-squares estimator actually attains a rate of $\BigOh{\frac{d + \log(1/\delta)}{T}}$ (without the log factors) and that our lower bounds are off by a factor of $\BigOhTil{\sqrt{d + \log(1/\delta)}}$.
%; we provide a heuristic argument which is specific to least squares in Appendix~\ref{sec:Lower_Bound_Loose}. It is possible that shaving off the $\BigOhTil{\sqrt{d + \log(1/\delta)}}$ factor for $|\rho|$ close to one may be informational-theoretically possible (that is, attainable by an algorithm other than OLS).

% \subsection{Estimation Rates with Known $\Bst$}
% In this section, we show that the learning rates can be improved when the system is driven by inputs $u_t$ when the control matrix $\Bst$ is known. For ease, we assume that the inputs are white noise $u_t \overset{i.i.d}{\sim} \mathcal{N}(0,\sigmaU^2 I)$. In this regime, note that we have the decomposition:
% \begin{eqnarray}\label{eq:recursion_eq}
% &&X_{t+1} = \sum_{s=1}^{t} \Ast^{t-s}(\noise_s+\Bst u_t)~~\text{ which implies that}~~\Exp[X_tX_t^{\top}] = \sigma^2\Gramm_t~ + \sigmaU^2 \GrammB_t~,\nonumber\\
% \text{where}&& \quad \GrammB_t := \sum_{s=1}^t \Ast^{t-s}\Bst\Bst^\top \Ast^{t-s}~.
% \end{eqnarray}
% With these definitions, we can now state an analogue of Theorem~\ref{stable_thm} for control inputs, which is again a consequence of Theorem~\ref{main_thm}:
% \begin{thm}\label{thm:stbl_input} Fix $\delta \in (0,1/2)$ and consider the linear dynamical system $X_{t+1} = \Ast X_t + \Bst u_t + \noise_t $, where $\Ast$ is a marginally stable matrix in $\R^{d\times d}$, $X_0 = 0$, $\noise_t  \iid \calN(0,\sigma^2 I)$, and $u_t \sim \mathcal{N}(0,\sigmaU^2 I)$.
% \end{thm}

\subsection{General Time Series with Linear Responses}
\label{sec:meta_thm}

    In this section, we consider a sequence of covariate-response pairs $(X_t,Y_t)_{t \ge 1}$, where $Y_t = \Ast X_t + \noise_t$, with $Y_t, \noise_t \in \R^n$, $X_t \in \R^d$, and $\Ast \in \R^{n \times d}$. The least squares estimator is then
    \begin{eqnarray}
    \ALS(T) &:= \arg\min_{A \in \R^{d \times d}} \sum_{t=1}^T \frac{1}{2}\|Y_{t} - A X_t\|_2^2 \:.
    \end{eqnarray}

    We let $\calF_{t} := \sigma(\noise_0,\noise_1,\dots,\noise_{t},X_1,\dots,X_t)$ denote the filtration generated by the covariates and noise process. Note then that $Y_t \in \calF_{t}$ but $Y_t \notin \calF_{t - 1}$. Further, we assume $\eta_t | \calF_{t-1}$ is mean-zero, and $\sigma^2$-sub-Gaussian (i.e., $\Exp[\exp(\lambda \eta_t) | \calF_t)] \le e^{\sigma^2 \lambda^2/2}$).
    In this setting, the $\OLS$ estimator is given by
    $\ALS(T) := \arg\min_{A \in \R^{n \times d}} \sum_{t=1}^T \frac{1}{2}\|Y_t - A X_t\|_2^2$.
    The linear dynamical systems sub-case is recovered from this general setting when $Y_t = X_{t+1}$.

%For concreteness, the $\OLS$ estimates for the two cases considered here are respectively
%
%
%    \begin{align}
%    \label{eq:ols}
%    \ALS(T) &:= \arg\min_{A \in \R^{n \times d}} \sum_{t=1}^T \frac{1}{2}\|Y_t - A X_t\|_2^2.
%    \end{align}

%	In this section, we analyze the performance of the $\OLS$ estimator when applied to time series of the form $Y_t = \Ast X_t + \noise_t$.
%
% Theorem~\ref{stable_thm} is a consequence of the result presented in this section.
    To capture the excitation behavior observed in the case of linear systems we introduce a general martingale small-ball condition which quantifies the growth of the covariates $X_t$.
	\begin{defn}[Martingale Small-Ball]\label{def:bmsb} Let $(Z_t)_{t \ge 1}$ be an $\{\calF_t\}_{t \ge 1}$-adapted random process  taking values in $\R$. We say $(Z_t)_{t\ge 1}$ satisfies the $(k,\nu,p)$-block martingale small-ball (BMSB) condition if, for any $j \ge 0$, one has $\frac{1}{k}\sum_{i=1}^k \Pr( |Z_{j+i}| \ge \nu | \calF_{j}) \ge p$ almost surely.  Given a process $(X_t)_{t \ge 1}$ taking values in $\R^d$, we say that it satisfies the $(k,\Gamsb,p)$-BMSB condition for $\Gamsb \succ 0$ if, for any fixed $w\in \calS^{d-1}$, the process $Z_t:= \langle w, X_t\rangle$ satisfies $(k,\sqrt{w^\top \Gamsb w},p)$-BMSB.
	\end{defn}

	Such a small-ball condition is necessary for establishing a high-probability lower bound on $\sigma_{\min}(\sum_{t=1}^T X_tX_t^\top) = \min_{\direc \in \calS^{d-1}} \sum_{t=1}^T \langle X_t, \direc \rangle^2$. The parameter $\Gamsb$ corresponds to the minimum eigenvalue of the Gramians $\Gramm_t$ considered in the case of linear systems,
    and measures how excitable the covariates $X_t$ are. As expected, the next result shows that a higher $\nu$ leads to faster statistical estimation.

\begin{thm}\label{main_thm} Fix $\epsilon, \delta \in (0,1)$, $T \in \N$ and $0 \prec \Gamsb \preceq \Gambar$. Then if $(X_t,Y_t)_{t \ge 1} \in (\R^d \times \R^n)^{T}$ is a random sequence such that (a) $Y_t = \Ast X_t + \noise_t$, where $\noise_t | \calF_t$ is $\sigma^2$-sub-Gaussian and mean zero, (b) $X_1,\dots,X_T$ satisfies the $(k,\Gamsb,p)$-small ball condition, and (c) such that $\Pr[\sum_{t = 1}^T X_t X_t^\top \npreceq T\Gambar] \le \delta$. Then if 
\begin{align}\label{eq:main_thm_condition}
T \geq  \frac{10k}{p^2}\left(\log \left(\frac{1}{\delta}\right) + 2d\log(10/ p) +  \log \det (\Gambar \Gamsb^{-1}) \right),
\end{align}
we have
\begin{align}
\Pr\left[\opnormbig{\ALS(T)-\Ast} >\frac{90\sigma}{p}\sqrt{\frac{n + d\log \frac{10}{p} + \log \det \Gambar \Gamsb^{-1} + \log\left(\frac{1}{\delta}\right)}{T\lambda_{\min}(\Gamsb)}} \right] \le  3\delta.
\end{align}
\end{thm}
%\begin{remark}
The proof of Theorem~\ref{main_thm} is outlined in Section~\ref{sec:mn_thm_proof}, and technical details are deferred to Appendix~\ref{app:main_thm}. We remark that the conclusion of Theorem~\ref{main_thm} still holds if one replaces the $(k,\Gamsb,p)$ small-ball condition with any high probability lower bound of the form $\Pr\left( \sum_{t = 1}^T X_t X_t^\top \succsim T \Gamsb \right) \leq \delta$. In this case, one does need to consider blocks of length $k$ in this case, and can thus dispense with the restriction in~\eqref{eq:main_thm_condition}.
%\end{remark}

\subsection{Analysis Techniques\label{sec:analysis_tech}}
Let $\ALS = \ALS(T)$, let $\matX \in \R^{T \times d}$ denote the matrix whose rows are $X_t$, and $\Noise \in \R^{T \times n}$ denote the matrix
whose rows are $\noise_t$.  Consider the compact SVD of $\matX$ and
$\matX = \matU \matSig \matV^\top$, where $\matSig,\matV \in \R^{d \times d}$
and $\matU \in \R^{T \times d}$. Note then that we have $\ALS - A_* =
(\matX^{\dagger} \Noise)^\top $ which implies that
\begin{eqnarray}\label{eq:error_decomp}
\opnorm{\ALS - A_\star}  = \opnorm{\matX^{\dagger} \Noise} \leq \sigma_{d}(\matX)^{-1}\opnorm{\matU^{\top}\Noise}.
\end{eqnarray}
Herem $\sigma_{d}((\matX)^{-1})$ denotes the $d$-th largest singalue value of $\matX$, which is precisely $\sqrt{1/\lambda_{\min}(\matX^\top \matX)}$.
The technical challenge arises from the fact that the singular space $\matU^{\top}$ and $\Noise$ are correlated, and that the rows $X_t$ of $\matX$ are also dependent. We upper bound $\opnorm{\matU^{\top}\Noise}$ with Lemma~\ref{eq:normalized}, a martingale-Chernoff bound that gives precise control on the deviations of sub-Gaussian martingale sequences in terms of random variance proxies. We explain this argument in more detail at the end of Section~\ref{sec:mn_thm_proof}.

Our lower bound on $\sigma_{\min}(\matX) = \sqrt{\lambda_{\min}(\sum_{t=1}^TX_tX_t^\top)}$ eschews mixing-time arguments in favor of a careful modification of Mendelson's small-ball method~\citep{mendelson14b}. %to lower bound the least singular value of the covariance matrix $\sum_{t=1}^T X_tX_t^\top$.
We divide our covariates into size-$k$ blocks $\{X_{(\ell-1) k + 1},\dots,X_{\ell k}\}$, such that for  any fixed $\direc \in \calS^{d-1}$, the quantity $\sum_{\ell=1}^k \langle X_{(\ell-1) k + 1}, \direc \rangle^2$ can be lower bounded by the $(k, \Gamsb, p)$-BMSB condition. Proposition~\ref{prop:Small_Ball} below (proved in Appendix~\ref{sec:proof_small_ball}) then implies that we have $\sum_{t=1}^T \langle X_t, \direc \rangle^2~\gtrsim~T w^\top \Gamsb w$~with probability at least $1 - \exp( - c T/k)$ for some constant $c$:
\begin{prop}\label{prop:Small_Ball} Suppose that $(Z_1,Z_2,\dots, Z_T) \in \R^T$ satisfies the $(k,\nu,p)$-BMSB condition. Then
    \begin{eqnarray}
    \Pr\left[\sum_{i=1}^T Z_i^2 \le \frac{\nu^2p^2}{8} k\floor{T/k}\right] \leq   e^{-\frac{\floor{T/k}p^2}{8}}.
    \end{eqnarray}
 \end{prop}
Once $T$ is large enough, these high-probability bounds can be used to derive a uniform bound over $\direc \in \calS^{d-1}$ via a discretization argument (Lemma~\ref{lem:eig_Packing_Lem}).  In general there is a trade-off between the size of the blocks $k$ and the probability guarantee obtained: a larger block size leads to a larger
parameter $\nu$ and a faster rate, but it degrades the probability guarantee. %e show that in the case of linear systems, this trade-off can be settled to yield fast rates of estimation for nearly unstable matrices $\Ast$ which hold with high probability.

  %	\maxs{explain: advantage is that we need failure probability $e^{-\BigOhTil{d}}$ for lb on $\sigma_{\min}(\dots)$, so we can have as few as $T/\BigOhTil{d}$-blocks. And fewer blocks means large $k$ means large $\nu$, necessarily bc we can only lower bound with $\nu$ w. const prob, not high prob. }

%!TEX root = LWM.tex
\section{Theorem~\ref{stable_thm} as a corollary of Theorem~\ref{main_thm}\label{sec:pf_stable_thm}}

In this section, we show how to obtain Theorem~\ref{stable_thm} as a fairly straightforward consequence of our meta-theorem, Theorem~\ref{main_thm}.
By assumption, our noise process satisfies the $\sigma^2$-sub-Gaussian tail condition. Moreover, we see that
\begin{eqnarray*}
\Pr[\matX^\top \matX \npreceq \frac{\sigma^2 d}{\delta}T\Gramm_T] &=& \Pr[\lambda_{\max}((T\Gramm_T)^{-1/2}\matX^\top \matX (T\Gramm_T)^{-1/2}) \ge \frac{d}{\sigma^2 \delta}] \\
&\le& \frac{\delta}{d\sigma^2} \cdot \Exp[\lambda_{\max}((T\Gramm_T)^{-1/2}\matX^\top \matX (T\Gramm_T)^{-1/2})]\\
&\le& \frac{\delta}{d\sigma^2} \cdot \Exp[\tr(T\Gramm_T)^{-1/2}\matX^\top \matX (T\Gramm_T)^{-1/2})] \le \delta
\end{eqnarray*}
where the last equality uses linearity of trace, expectation and $\Exp[\matX^\top \matX] = \sigma^2 \sum_{t=1}^T \Gamma_t \preceq \sigma^2 T\Gramm_T$. Hence, we can set $\Gambar = \Gramm_T$. Noting $\log \det (  (\frac{d}{\delta}\sigma^2\Gamma_T))(\sigma^2 \Gamma_{\floor{k/2}})^{-1}) = d\log d/\delta + \log \det (\Gamma_T \Gamma_{\floor{k/2}}^{-1})$, it suffices to verify that $(X_t)_{t \ge 1}$ satisfies the $\left(k, \sigma^2 \Gamma_{\floor{k/2}}, \frac{3}{20}\right)$-BMSB condition:

\begin{prop}\label{prop:linear_system_small_ball_full} Consider the linear dynamical system $X_{t+1} = AX_t + \noise_t $, where $X_0 \in \R^d$ and $\noise_t  \iid \calN(0,\sigma^2 I)$, and let  $\Gramm_t :=  \sum_{s=0}^{t-1} (A^{s})(A^s)^{\top}$. Then, for $1 \leq k \leq T$, the process $( X_t )_{t \ge 1}$ satisfies the
\begin{eqnarray}
\left(k, \sigma^2 \Gramm_{\floor{k/2}}, \frac{3}{20}\right)\text{-block martingale small-ball condition.}
\end{eqnarray}
\end{prop}
\begin{proof}
Let $1\leq k' \leq k$. Note that, for $t \ge 1$, $X_{s+t} \big{|} \calF_{s} \sim \mathcal{N}(\langle w, A^t X_0\rangle, \sigma^2 w^\top \Gramm_t w)$. Hence, for $t \ge k'$, one has 
\begin{eqnarray*}
\Pr\left(|\langle w, X_{s+t}\rangle| \geq \sigma \sqrt{w^\top \Gramm_{k'}w} \big{|} \calF_s \right) \overset{(i)}{\ge} \Pr\left(|\langle w, X_t\rangle| \geq \sigma \sqrt{w^\top \Gramm_{t}w}  \big{|} \calF_s\right) 
\overset{(ii)}{\ge} \frac{3}{10},
\end{eqnarray*}
where $(i)$ uses the fact that $\Gramm_t \succeq \Gramm_{k'}$ for $t \ge k'$, and $(ii)$ follows from the Paley-Zygmund lower bound,
\begin{eqnarray}\label{eq:paley_zygmund}
\forall t \in \R, \quad \Pr_{Z \sim \calN(0,\sigma^2)}[|t+Z| \ge \sigma]  \ge \Pr[|Z| \geq \sigma] \geq 3/10 \:.
\end{eqnarray}
Therefore, we have
\begin{eqnarray*}
\frac{1}{k} \sum_{t = 1}^k \Pr\left(|\langle w, X_t\rangle| \geq \sigma \sqrt{w^\top \Gramm_{k'}w}\right) ~\geq~  \frac{1}{k} \sum_{t = k'}^k \Pr\left(|\langle w, X_t\rangle| ~\geq~ \sigma \sqrt{w^\top \Gramm_{k'}w}\right)
\geq  \frac{3}{10}\frac{k - k' + 1}{k}.
\end{eqnarray*}
Finally, choosing $k' = \floor{k/2}$ yields the desired conclusion.
\end{proof}

\iffalse
\subsection{Bounds for 'Purely Unstable' Stables \maxs{Stephen:}}

\horia{A combination of this analysis with Stephen's analysis would yield linear convergence in the purely unstable case if
\begin{align*}
N \geq C (d + \log (1/\delta))
\end{align*}
independent rollouts are used}
\fi

%%% Local Variables:
%%% mode: latex
%%% TeX-master: "LWM"
%%% End:

%!TEX root = LWM.tex

%!TEX root = LWM.tex

\section{Proof of Theorem~\ref{main_thm}\label{sec:mn_thm_proof}}

Again, we let
$\matX \in \R^{T \times d}$ denote the matrix whose rows are $X_t$, and $\Noise \in \R^{T \times n}$ denote the matrix
whose rows are $\noise_t$, and consider the compact SVD of $\matX$ and
$\matX = \matU \matSig \matV^\top$, where $\matSig,\matV \in \R^{d \times d}$
and $\matU \in \R^{T \times d}$. Recalling~\eqref{eq:error_decomp}, we have $\opnorm{\ALS(T) - A_\star} \le \sigma_{d}(\matX)^{-1}\opnorm{\matU^{\top}\Noise}.$
Let $K$ denote a threshold parameter to be chosen later.
Then~$\opnorm{\ALS(T) - A_\star} \le \sigma_{d}(\matX)^{-1}\opnorm{\matU^{\top}\Noise}$ implies the following set-theoretic inclusions,
\begin{align*}
  &\left\{ \opnorm{\matX^{\dagger} \Noise} \geq \frac{4K}{p\sqrt{k\floor{T/k}\lambda_{\min}(\Gamsb) }} \right\} \cap \left\{ \matX\matX^{\top} \succeq \frac{k \floor{T/k} p^2 \Gamsb}{16} \right\} \\
  &\qquad \subseteq \left\{ \opnorm{\matU^{\top} \Noise} \geq \frac{4K}{p\sqrt{k\floor{T/k}\lambda_{\min}(\Gamsb) }} \sigma_{\min}(\matX)  \right\}\cap \left\{ \matX\matX^{\top} \succeq \frac{k \floor{T/k} p^2 \Gamsb}{16} \right\} \\
  &\qquad \subseteq \left\{ \opnorm{\matU^{\top} \Noise} \geq K \right\} \cap \left\{ \matX\matX^{\top} \succeq \frac{k \floor{T/k} p^2 \Gamsb}{16} \right\} \:.
\end{align*}
Now define the following events
\begin{align*}
  \calE_1 := \left\{ \opnorm{\matU^{\top} \Noise} \geq K \right\} \:, \quad \calE_2 := \left\{\matX^\top \matX \succeq \frac{k\lfloor T/k\rfloor p^2 \Gamsb }{16}\right\} \:, \quad \calE_3 := \left\{ \matX^\top \matX  \npreceq \Gambar \right\} \:.
\end{align*}

Then we have
\begin{align*}
  &\Pr\left[ \opnorm{\ALS(T) - A_\star} \geq \frac{4K}{p\sqrt{k\floor{T/k} \lambda_{\min}(\Gamsb) }} \right] \quad\leq \quad\Pr\left[ \left\{\opnorm{\matX^\dagger \Noise} \geq \frac{ 4K}{p\sqrt{k\floor{T/k}\lambda_{\min}(\Gamsb) }} \right\} \right] \\
  &\qquad\leq \Pr\left[ \left\{\opnorm{\matX^\dagger \Noise} \geq \frac{ 4K}{p\sqrt{k\floor{T/k}\lambda_{\min}(\Gamsb) } } \right\} \cap \calE_3^c \right] + \Pr[ \calE_3 ] \\
  &\qquad\leq \Pr\left[ \left\{\opnorm{\matX^\dagger \Noise} \geq \frac{ 4K}{p\sqrt{k\floor{T/k}\lambda_{\min}(\Gamsb) }} \right\} \cap \calE_2 \cap \calE_3^c \right] + \Pr[ \calE_2^c \cap \calE_3^c] + \Pr[ \calE_3 ] \\
  &\qquad\leq \Pr\left[ \calE_1 \cap \calE_2 \cap \calE_3^c \right] + \Pr[ \calE_2^c \cap \calE_3^c] + \Pr[ \calE_3 ] \:.
  %&\leq& \underbrace{\Pr\left[ \left\{\opnorm{\matU^\top \mateps} \geq K \right\} \cap \calE_1 \right]}_{T_1} + \underbrace{\Pr\left[ \calE_1^c \cap \left\{\sigma_{\max}(\matX)^2 \leq \itlamax\right\}\right]}_{T_2} + \underbrace{\Pr\left[\sigma_{\max}(\matX)^2 \geq \itlamax\right]}_{T_3}~.
\end{align*}
By assumption $\Pr[\calE_3] \leq \delta$. Our task is to show that both $\Pr\left[ \calE_1 \cap \calE_2 \cap \calE_3^c \right]$ and $\Pr[ \calE_2^c \cap \calE_3^c]$ are upper bounded by $\delta$ for the choice of $K = 20\sigma \sqrt{n + d\log \left(1 + \frac{32\sqrt{\itlamax}}{\sqrt{k\floor{T/k}} \nu p}\right) + \log(3/\delta) }$. Both bounds are proven in detail in Appendix~\ref{app:main_thm}, but here we state the main technical arguments required for their proof. All supplementary technical results (Lemma~\ref{lem:eig_Packing_Lem}, Lemma~\ref{lem:martingale_lem} and Proposition~\ref{prop:Small_Ball}) are proven in Appendix~\ref{sec:proof_technical}.

Our bound on $\Pr[ \calE_2^c \cap \calE_3^c]$ comes directly from the BMSB assumption. Applying Proposition~\ref{prop:Small_Ball}, we have that for any fixed $\direc \in \calS^{d-1}$, $ \Pr\left[\sum_{t=1}^T \langle \direc, X_t \rangle^2 \le \frac{\direc^\top\Gamsb \direc p^2}{8} k\floor{T/k}\right] \leq   e^{-\frac{\floor{T/k}p^2}{8}}$. To obtain a Lowner lower-bound $\matX^\top\matX$, we need to strengthen the above pointwise bound into a lower bound on $\inf_{\direc \in \calS^{d-1} }\sum_{t=1}^T \langle \direc, X_t \rangle^2 $. This is achieved through the following covering lemma, proved in Appendix~\ref{sec:proof_lem_eig_pack}:
  \begin{lem}\label{lem:eig_Packing_Lem}
  Let $Q \in \R^{T \times d}$ and consider matrices $0 \prec \Gamin \preceq \Gamax \in R^{d \times d}$. Then $\calT$ be a $1/4$-net of $\sphereGamin$ in $\normGamax$. Then if $\inf_{\direc \in \calT} w^{\top}Q^\top Qw \ge 1$ and $Q^\top Q \preceq \Gamax$, then 
  \begin{eqnarray}
  Q^\top Q \succeq \Gamin/2~.
  \end{eqnarray}
  \end{lem}
  Choosing $Q = \matX\matX^{\top}$, this lemma gives us a bound on the granularity at which we need to cover $\calS^{d-1}$ in terms of a uniform Lowner upper bound $\Gamax = T\Gambar$, and pointwise Lowner lower bound with $\Gamin = \frac{\Gamsb p^2}{8} k\floor{T/k}$. The details are worked through in Appendix~\ref{app:main_sigm_min}. Crucially, even though $\Gamax$ may be much larger than $\Gamin$ in a Lowner sense, this only enters logarithmically into our final bound via the relative volume $\log \det (\Gamax \Gamin^{-1})$ of the ellipsoids induced by $\Gamax$ and $\Gamin$.

  Lastly, we bound $\Pr\left[ \calE_1 \cap \calE_2 \cap \calE_3^c \right]$ which again is the event that $\left\{ \opnorm{\matU^{\top} \Noise} \geq K \right\} $, under the event that the spectrum of $\matX$ is bounded in some desired range. $\matU$ is difficult to control directly, so we instead work with quantities in terms of $\matX$. Simple linear algebra lets us write $\opnorm{\matU^{\top} \Noise} = \sup_{v \in \calS^{n-1}, \direc \in \R^{d}} \frac{\direc^{\top}\matX \Noise v}{\|\matX\direc\|}$.  The key idea here now is to use a martingale-Chernoff bound to show that, for any fix $\direc \in \calS^{d-1}$ and $v \in \calS^{n-1}$, either $\direc^\top \matX^\top \Noise v$ concentrates like a $\sigma^2 \cdot \|\matX \direc\|_2^2$-sub-Gaussian random variable, or $\direc^\top \matX^\top \Noise v $ is much smaller than the lower bound on $\sigma_{\min}(\matX)$ under $\calE_2 \cap \calE_3^c$.

  We emphasize that our bound controls $\direc^\top \matX^\top \Noise v$ in terms of the \emph{random} variance proxy is $\sigma^2 \cdot \|\matX \direc\|_2^2$. This is is subtle yet powerful, because it yields an immediate cancellation between the numerator and denominator of $ \frac{\direc^{\top}\matX \Noise v}{\|\matX\direc\|}$, implying in particular than $\Pr( \{\frac{\direc^{\top}\matX \Noise v}{\|\matX\direc\|} \gtrsim \log(1/\delta)\} \cap \calE_2 \cap \calE_3^c ) \lesssim \log (1/\delta)$. This lets us reduce our problem to finding an appropriate discretization (see Lemma~\ref{lem:Lin_packing}). We stress that an approach which bounds $\direc^{\top}\matX \Noise v$ and $\|\matX\direc\|$ separately would be considerably less sharp, and would degrade for slower-mixing systems. Our data-dependent concentration bound is a consequence of the following technical lemma, which we apply with $Z_t := \direc^\top X_t$ and $\noiseb_t = \noise_{t}^\top v$:
\begin{lem}\label{lem:martingale_lem} Let $\{\calF_{t}\}_{t \ge 0}$ be a filtration, and $\{Z_t\}_{t\ge 1}$ and $\{\noiseb_t\}_{t\ge 1}$ be real-valued processes adapted to $\calF_{t}$ and $\calF_{t + 1}$ respectively. Moreover, assume $\noiseb_t | \calF_{t}$ is mean zero and $\sigma^2$-sub-Gaussian.  Then, for any positive real numbers $\alpha$, $\beta$, $\beta_-$, $\beta_+$ we have
\begin{eqnarray}\label{eq:normalized}
&\mbox{(a) }& \Pr\left[ \left\{\sum_{t = 1}^T Z_t \noiseb_t  \ge \alpha\right\} \cap \left\{\sum_{t = 1}^T Z_t^2  \le \beta \right\}\right] \le \exp\left( - \frac{\alpha^2}{2 \sigma^2 \beta}\right).\\
&\mbox{(b) }& \Pr\left[\left\{\frac{\sum_{t=1}^T Z_t \noiseb_t}{\sqrt{\sum_{t=1}^T Z_t^2}} > \alpha\right\} \cap \left\{ \sum_{t=1}^TZ_t^2 \in [\beta_-,\beta_+]\right\}\right] \le \log{\left\lceil\frac{\beta_+}{\beta_-}\right\rceil}\exp\left(\frac{-\alpha^2}{6 \sigma^2}\right).
\end{eqnarray}
  \end{lem}

%%% Local Variables:
%%% mode: latex
%%% TeX-master: "LWM"
%%% End:

%!TEX root = LWM.tex
\section{Discussion and future work}
In this paper, we analyzed the the performance of the $\OLS$ estimator for the estimation of linear dynamics $X_{t + 1} = \Ast X_t + \noise_t$ from a single trajectory $X_0, X_1, \ldots, X_T$, as a special case of linear estimation in time series. We show that, up to logarithmic factors, the $\OLS$ estimator attains an information-theoretic lower bound for $\rho(\Ast) < 1$, provided that $T \gtrsim \frac{d}{1-\rho(\Ast)}$. Moreover, we present an analysis that eschews both mixing and concentration arguments for estimation in time series. We believe that there are several promising directions for future work:
\begin{itemize}
\item Our lower and upper bounds do not perfectly match, even when $\rho(\Ast) < 1$. We believe resolving these indiscrepancies may shed greater insight into learning in dynamical systems.
\item While our analysis can accomodate an unknown $\Bst$, the rates do not distinguish between the error in the estimation of $\Ast$ and that of $\Bst$. In future, we hope to develop sharp error rates for $\Ast$ and $\Bst$ individually, similar to \cite{dean17} in the independent covariates setting.
\item While our guarantees are stated in the operator norm, control applications may require more granular notions of error which vary for different modes of $\Ast$. Developing error bounds which capture the error rate at each mode may result in more applicable bounds for control applications downstream.
\item Our convergences rates degrade for systems with $\rho(\Ast) > 1$, whereas we know from~\cite{faradonbeh17a} that these systems are still identifiable with $\OLS$. Is there a unified analysis for systems with stable and unstable modes?
\item In many systems, we do not observe $X_t$ directly, but only view $CX_t$ for a short matrix $C \in \R^{n_o \times n}$, where $n_0 \le n$. \cite{hazan17,hazan18} provide filtering techniques to minimize regret for diagonalizable matrices; it would be interesting to understand the sample complexity for estimating arbitrary matrices with these limited observations.
\item Ultimately, we would like to understand what sequences of control inputs $u_t$ yield the most accurate estimation of the system $(\Ast,\Bst)$. This would inform adaptive algorithms which adjust the sequence $u_t$ in a sequential fashion, and online algorithms which ensure low regret relative to a given cost functional over time.
\end{itemize}

\bibliography{LWM}
\bibliographystyle{plainnat}
\bibpunct{(}{)}{;}{a}{,}{,}

\appendix
%!TEX root = LWM.tex
\section{Existence of $k$ for $\rho(\Ast) \leq 1$\label{app:consistent}}
\label{sec:appendix:kst}

\begin{prop}\label{prop:det_bound} Let $\Ast = SJS^{-1}$ and $T \ge d$, where $J$ has block sizes $b_{1},\dots,b_L$. Then,
\begin{eqnarray} 
 \log \det(\Gamma_T \Gamma_{k}^{-1}) \le 2d\log \cond(S) + d \log \frac{T}{k} +  4\log T \sum_{\ell: b_{\ell} \ge 2} b_{\ell}^2 
\end{eqnarray}
where $\cond(S)$ denotes the complex condition number of $S$, namely $\sqrt{\lambda_{\max}(S^*S)/\lambda_{\min}(S^*S)}$.
\end{prop}

The above proposition directly implies consistency for whenever $\rho(A_*) \le 1$:
\begin{cor}\label{cor:consistent}[Consistency of Least Squares] There exists a universal constants $C,c > 0$ such that for any time horizon $T$, $\delta \in (0,1/2)$, and any $\Ast$ with $\rho(A_*) \le 1$ and Jordan decomposition $\Ast = SJS^{-1}$, where $J$ has block sizes $b_{1},\dots,b_L$, $k \in \mathbb{N}$ satisfies the conditions of Theorem~\ref{stable_thm} provided that
\begin{eqnarray}\label{eq:T_over_K_jordan}
\frac{T}{k} \ge c\left( d \log \left(\frac{d\cond(S)T}{k \delta }\right) + \log T \sum_{\ell: b_{\ell} \ge } b_{\ell}^2  \right)~,
\end{eqnarray}
Taking $k =1$, implies a minimax-rate of estimation of 
\begin{eqnarray}\label{eq:minmax_rate}
\|A - A^*\|_{\op} \le C\sqrt{ \frac{  \log (dT\cond(S)/\delta) + \log T \sum_{\ell: b_{\ell} \ge 2} b_{\ell}^2 }{T}}
\end{eqnarray}
which holds as long as $T \ge c(d \log (\cond(S)\frac{d}{k \delta } + \sum_{\ell: b_{\ell} \ge 1} b_{\ell}^2 \log (\sum_{\ell: b_{\ell} \ge 1} b_{\ell}^2))$
\end{cor}
\begin{rem}\label{rem:consistent} Observe that if the Jordan decomposition is such that $\sum_{\ell: b_{\ell} \ge 1} b_{\ell}^2 \lesssim d$, then our minimax rate coincides with the minimax rate of linear regression with isotropic covariates up to logarithmic factors. Moreover, we above that 
\end{rem}
For diagonalizable matrices, the rates can be made more explicit:
\begin{cor}\label{cor:diag}
There exists a universal constants $C,c > 0$ such that the following holds. Fix any time horizon $T$, $\delta \in (0,1/2)$, let $\Ast =  SDS^{-1}$ be diagonalizable with $\rho(A_*) \le 1$ and minimum eigenvalue-magnitude $\underline{\rho}$. Then,  $k \in \mathbb{N}$ satisfies the conditions of Theorem~\ref{stable_thm} provided that
\begin{eqnarray}\label{eq:T_over_K_diag}
\frac{T}{k} \ge c d \log (\cond(S)/\delta) 
\end{eqnarray}
This implies that for $T \ge c d \log (\cond(S)/\delta)$
\begin{eqnarray}
\|A - A^*\|_{\op} &\le& C\sqrt{ \frac{ d \log (\cond(S)/\delta)  }{T \Gamma_{\floor{T/c d \log (\cond(S)/\delta)}}}} \label{eq:gamma_diag_bound}\\
&\le& \sqrt{ \frac{d \log (\cond(S)/\delta)  }{T (1 + \cond(S)^{-2}\sum_{s=1}^{\floor{T/c d \log (\cond(S)/\delta)}-1 } \underline{\rho}^{2s}  }} \label{eq:no_gamma_diag_bound}
\end{eqnarray}
\end{cor}

\subsection{Proof of Corrollary~\ref{cor:consistent}}

By Theorem~\ref{stable_thm}, $T$ and $k$ must satisfy the inequality
\begin{eqnarray}\label{eq:cond_burnin}
T/k \ge c(d \log \frac{d}{\delta} + \log \det \Gamma_{T}\Gamma_{k}^{-1}) 
\end{eqnarray}
Equation~\eqref{eq:T_over_K_jordan} follows directly from Proposition~\ref{prop:det_bound}. Specializing to $k = 1$, Proposition~\ref{prop:det_bound} and Theorem~\ref{stable_thm} immediately imply \eqref{eq:minmax_rate}. To upper bound the burn-in time for $T$, we note that by Proposition~\ref{prop:det_bound}, the condition~\eqref{eq:cond_burnin} holds as soon as
\begin{eqnarray}\label{eq:diag_suff}
T \ge c'(d \log \frac{d\cond(S)}{\delta}) \text{ and } T \ge c' \log T(d+ \sum_{\ell:b_{\ell} \ge 2}b_{\ell}^2) 
\end{eqnarray}
for a universal constant $c'$. We can bound $d+ \sum_{\ell:b_{\ell} \ge 2}b_{\ell}^2 \le \sum_{\ell} b_{\ell}^2$, where the latter sum is over all Jordan blocks. We now invoke the following lemma, which we prove shortly:
\begin{lem}\label{alpha:lem}  Let $\alpha \ge 1$. Then for any $T \in \mathbb{N}$, $T \ge \alpha \log T$ as soon as $T \ge 2\alpha \log 4\alpha$
\end{lem}
The lemma implies that it is enough to ensure $T \ge c'(d \log \frac{d\cond(S)}{\delta})$ and that $T \ge 2c' \log T( 4\sum_{\ell} b_{\ell}^2) $, both of which can be ensured by choosing the constant $c$ in Corollary~\ref{cor:consistent} to be sufficiently large.

\begin{proof}[Proof of Lemma~\ref{alpha:lem}]
Taking derivatives, $T \mapsto T - \alpha \log T$ is increasing in $T$ for all $T \ge \alpha$. Hence, it suffices to show that for $T = 2\alpha \log 4\alpha$, $T \ge \alpha \log T$. Observe that for this choice of $\alpha$,
\begin{eqnarray*}
\alpha \log T = \alpha \log (2\alpha \log (4 \alpha) )\\
&\le& \alpha \log ((2\log 4)\cdot \alpha   + 2\alpha^2 )\\
&\le& \alpha \log ( (2 \log 4 + 2) )  \alpha^2) \quad \text{since } \alpha \ge 1\\
&=& 2\alpha \log \sqrt{ 2 \log 4 + 2} \le 2 \alpha \log 4~.
\end{eqnarray*}
\end{proof}

\subsection{Proof of Corrollary~\ref{cor:diag}}
By Theorem~\ref{stable_thm}, $T$ and $k$ must satisfy the inequality
\begin{eqnarray*}
T/k \ge c(d \log \frac{d}{\delta} + \log \det \Gamma_{T}\Gamma_{k}^{-1}) 
\end{eqnarray*}
Using the upper bound on $\log \det \Gamma_T \Gamma_k^{-1}$ from Proposition~\ref{prop:det_bound}, it is enough to ensure that 
\begin{eqnarray}\label{eq:diag_stuff}
\frac{T}{k} \ge c'(d \log \frac{d\cond(S)}{\delta}) \text{ and } \frac{T}{k} \ge c' d \log \frac{T}{k}
\end{eqnarray}
for some universal constant $c'$ (note that the term $\sum_{\ell: b_{\ell} \ge 2} b_{\ell}^2$ vanishes for diagonalizable $A_*$). By inflating $c'$, we may assume $c' \ge 1$. Appling Lemma~\ref{alpha:lem} with change of variables $T \leftarrow T/k$,~\eqref{eq:diag_stuff} holds as long as as long as $T/k \ge 2c' d \log(4 c' d)$, and $T/k \ge c'(d \log \frac{d\cond(S)}{\delta})$, which holds as long as
\begin{eqnarray*}
T/k \ge c'''(d \log \frac{d\cond(S)}{\delta})) 
\end{eqnarray*}
for some constant $c'''$. This proves~\eqref{eq:T_over_K_diag}. We then see that that~\eqref{eq:gamma_diag_bound} in Corollary~\ref{cor:diag} is an immediate consequence of~\ref{cor:diag}, and~\eqref{eq:gamma_diag_bound} follows from the Lowner Lower bound, with $\Ast = SDS^{-1}$:
\begin{eqnarray*}
\Gamma_T &=& I + \sum_{t=1}^{T-1} (\Ast^t)(\Ast^t)^{*}\\
&=& I + \sum_{t=1}^{T-1} (SD^tS^{-1})(SD^tS^{-1})^{*}\\
&=& I + S\sum_{t=1}^{T-1} SD^tS^{-1}S^{-*}D^{t*}S^{*}\\
&\succeq& I + \lambda_{\min}(S^{-1}S^{-*})\sum_{t=1}^{T-1} SD^tD^{t*}S^{*}\\
&\succeq& I + \lambda_{\min}(S^{-1}S^{-*})SS^{*}(\sum_{t=1}^{T-1} \underline{\rho}^{2t})\\
&\succeq& I + \lambda_{\min}(S^{-1}S^{-*})\lambda_{\max}(SS^{*})(\sum_{t=1}^{T-1} \underline{\rho}^{2t})I\\
&=& I(1+\cond(S)^{-2}(\sum_{t=1}^{T-1} \underline{\rho}^{2t}))
\end{eqnarray*}

\subsection{Proof of Proposition~\ref{prop:det_bound}}
Let $\Ast = SJS^{-1}$, where $J$ is a Jordan-Block matrices with blocks $J_1,\dots,J_L$ of sizes $b_1,\dots,b_L$. Note that even those $\Ast$ is real, $S$ and $J$ may be complex valued, so we shall use adjoints instead of transposes. We can compute 
\begin{eqnarray*}
\Gamma_{t} =  S\sum_{s=0}^{t-1} (J^s)S^{-1}S^{-\adj}J^{s\adj}S^{\top}
\end{eqnarray*}
Hence, 
\begin{eqnarray*}
&&\log \det (\Gamma_T \Gamma_k^{-1}) \\
&=&  \log \det \{S\sum_{s=0}^{T-1} J^sS^{-1}S^{-\adj}J^{s\adj}S^\top(S\sum_{s=0}^{k-1} J^sS^{-1}S^{-\adj}J^{s\adj}S^\adj)^{-1} \}\\
&=&  \log \det \{S\sum_{s=0}^{T-1} J^sS^{-1}S^{-\adj}J^{s\adj}(\sum_{s=0}^{k-1} (J^s)S^{-1}S^{-\adj}J^{s\adj})^{-1}S^{-1} \}\\
&=&  \log \det \{\sum_{s=0}^{T-1} J^sS^{-1}S^{-\adj}J^{s\adj}(\sum_{s=0}^{k-1} (J^s)S^{-1}S^{-\adj}J^{s\adj})^{-1} \}\\
\end{eqnarray*}
Lower bounding $S^{-1}S^{-\adj} \lesssim \sigma_{\min}(S)^{-1}$ and $S^{-1}S^{-\adj} \gtrsim 1/\sigma_{\max}(S)^2$, we can upper bound the above by
\begin{eqnarray*}
\log \det (\Gamma_T \Gamma_k^{-1}) &\le& \log \det \{\cond(S)^2\sum_{s=0}^{T-1} J^sJ^{s\adj}(\sum_{s=0}^{k-1} (J^s)J^{s\adj})^{-1} \} \\\\
&=&2 d \log \cond(S) + \log \det \{\sum_{s=0}^{T-1} (J^s)J^{s\adj}(\sum_{s=0}^{k-1} J^sJ^{s\adj})^{-1} \} \\
\end{eqnarray*}
To continue the bound, write the Jordan matrices $J = \mathrm{block}(J_1,\dots,J_L)$ as block diagonal matrices. Then 
\begin{eqnarray*}
\log \det (\Gamma_T \Gamma_k^{-1}) &\le& \log \det \{\cond(S)^2\sum_{s=0}^{T-1} (J^s)J^{s\adj}(\sum_{s=0}^{k-1} J^sJ^{s\adj})^{-1} \} \\\\
&=&2 d \log \cond(S) \sum_{b=1}^B \log \det \{\sum_{s=0}^{T-1} J_{\ell}^sJ_{\ell}^{s\adj}(\sum_{s=0}^{k-1} J_{\ell}^sJ_{\ell}^{s\adj})^{-1} \} \\
\end{eqnarray*}
If $J_{\ell} = a_{\ell}$ is a Jordan matrix with block size equal to $1$, then 
\begin{eqnarray*}
\log \det \{\sum_{s=0}^{T-1} (J_{\ell}^s)J_{\ell}^{s\adj}(\sum_{s=0}^{k-1} J_{\ell}^sJ_{\ell}^{s\adj})^{-1} \} &=& \log \frac{\sum_{s=0}^{T-1} |a_{\ell}|^{2s}}{\sum_{s=0}^{k-1}|a_{\ell}|^{2s}}\\
&\le& \log \frac{\ceil{T/k}\sum_{s=0}^{k-1} |a_{\ell}|^{2s}}{\sum_{s=0}^{k-1}a_{\ell}^{2s}} = \log(\ceil{T/k})\\
\end{eqnarray*}
where the inequality uses the fact that $a_{\ell}^{2s}$ is decreasing. If $\dim (J_{\ell}) > 1$, we shall use the following lemma:
\begin{lem}\label{dom_lem} Let $A \succeq 0$ be a $d \times d$ Complex Hermitian matrix. Then $A \preceq d \mathrm{Diag}(A)$, where $\mathrm{Diag}(A)$ is the diagonal matrices whose diagonal entries are those of $A$. 
\end{lem}
\begin{proof}
We can write
\begin{eqnarray*}
d\mathrm{Diag}(A) - A &=&  (d-1)\mathrm{Diag}(A) +  \sum_{1 \le i \ne j \le d} A_{ij}e_ie_j^{\top} + \overline{A_{ij}}e_je_i^{\top}\\
&=&   \sum_{1 \le i \ne j \le d} A_{ii}e_{i}e_i^\adj + A_{ij}(e_ie_j^{\top} + e_je_i^\top) + A_{jj}e_{j}e_j^\top \succeq 0
\end{eqnarray*}
\end{proof}
We can then bound 
\begin{eqnarray*}
\log \det \{\sum_{s=0}^{T-1} J_{\ell}^sJ_{\ell}^{s\adj}(\sum_{s=0}^{T-1} (J_{\ell}^s)J_{\ell}^{s\adj})^{-1} \} &\overset{(i)}{\le}& \log \det \{\sum_{s=0}^{T-1} J_{\ell}^sJ_{\ell}^{s\adj} \} \\
&\overset{(ii)}{\le}& \log \det \{\dim(J_{\ell})\mathrm{Diag}(\sum_{s=0}^{T-1} (J_{\ell}^s)J_{\ell}^{s\adj}) \} \\\\
&=& \dim(J_{\ell})\log \dim(J_{\ell}) + \sum_{i=1}^{\dim(J_{\ell})} \log (\sum_{s=0}^{T-1}(J_{\ell}^sJ_{\ell}^{s\adj})_{ii} ) \\
&=&\dim(J_{\ell})\log \dim(J_{\ell}) + \sum_{i=1}^{\dim(J_{\ell})} \log (\sum_{s=0}^{T-1}\sum_{j}(J_{\ell}^s)_{ij}^2) 
\end{eqnarray*}
where $(i)$ uses that $\sum_{s=0}^{T-1} J_{\ell}^sJ_{\ell}^{s\adj} \succeq I$, and $(ii)$ uses Lemma~\ref{dom_lem}. We can then compute that if $J_{\ell}$ has diagonals $a_{\ell}$,
\begin{eqnarray*}
(J_{\ell}^s)_{i,j} = \begin{cases} \binom{s}{j - i}a_{\ell}^{s - (j-i)\vee 0} & i \le j \\
0 & \text{ otherwise} \end{cases}
\end{eqnarray*}
So that 
\begin{eqnarray*}
\sum_{j}(J_{\ell}^s)_{ij}^2 &=& \sum_{j \ge 1}(\binom{s}{j - i})^2|a_{\ell}|^{2(T - (j-i)\vee 0)}\\
&=&  \dim(J_\ell)^2 s^{2(\dim(J_\ell) - i)}
\end{eqnarray*}
Hence, 
\begin{eqnarray*}
\sum_{i=1}^{\dim(J_{\ell})} \log (\sum_{s=0}^{T-1}\sum_{j}(J_{\ell}^s)_{ij}^2) &\le& \sum_{i=1}^{\dim(J_{\ell})} \log( \dim(J_{\ell})^2 \sum_{s=0}^{T-1} s^{2(\dim(J_{\ell}) - i)} )\\
&\le& \sum_{i=1}^{\dim(J_{\ell})} \log(\dim(J_{\ell})^2  T^{2(\dim(J_{\ell}) - i) + 1}) \\
&=& 2\dim(J_{\ell})\log (\dim(J_{\ell})) + \sum_{i=1}^{\dim(J_{\ell})} 2(\dim(J_b) - i)+1\log(T)  \\
&=& 2\dim(J_{\ell})\log (\dim(J_{\ell})) + \log(T) \cdot\sum_{i=1}^{\dim(J_{\ell})} i   \\
&\le& 2\dim(J_{\ell})\log (\dim(J_{\ell})) + \dim(J_{\ell})^2 \log T  \\
&\le& 4 \dim(J_{\ell})^2 \log T  \\
\end{eqnarray*}
where the last line uses that $T \ge d \ge \dim(J_{\ell})$, and that $\dim(J_{\ell}) \ge 2$.

\section{Specialized Analysis in Scalar Linear Systems}
\label{sec:1d_appendix}

In this appendix, we present specialized upper and lower bounds in the case of scalar systems. Specifically, we consider $x_{t+1} = \ast x_t + \noise_t$, where $\noise_t \sim \mathcal{N}(0,\sigma^2)$, and $x_0 = 0$. Our upper bound has sharp, explicit constants, and captures the correct qualitative behavior for unstable scalar systems:
\begin{thm}\label{thm:one_d_thm}
Let $\epsilon \in (0,1)$ and $\delta \in (0, 1/2)$. Then $\Pr[|\als(T) - \ast | \leq \epsilon] \ge 1-\delta$ as long as
\begin{align*}
T \geq \begin{cases}\frac{8}{\epsilon}\log\left(\frac{2}{\delta}\right) + \frac{4}{\epsilon^2}(1 - (|\ast| - \epsilon)^2)\log\left(\frac{2}{\delta}\right) & \ast \le 1 + \epsilon \\
\max\left\{\frac{8}{(|\ast| - \epsilon)^2 - 1}\log\left(\frac{2}{\delta}\right), \frac{4\log(\frac{1}{\epsilon})}{\log(|\ast| - \epsilon)} + 8\log\left(\frac{2}{\delta}\right)\right\} & \ast > 1 + \epsilon~.
\end{cases}
\end{align*}
\end{thm}
We match the upper bound with a lower bound which shows that our rates are optimal. Unlike the $d$-dimension case, our lower bound considers ``local alternatives'' rather than scaled orthogonal matrices\footnote{In one dimension, the orthogonal matrices are just the set $\{-1,1\}$, and this precludes packing `nearby' orthogonal matrices as in the $d$-dimensional case}:
\begin{thm}[1-D Lower Bound]\label{thm:info_lb_1d}  Fix an $\ast \in \R$, and define $\Gramm_T := \sum_{t=1}^{T}\ast^{2t}$. Fix an alternative $a' \in \{\ast - 2\epsilon,\ast + 2\epsilon\}$, and $\delta \in (0,1/4)$. Then for any estimator $\widehat{a}$,
\begin{eqnarray*}
\sup_{a \in \{a^*,a'\}}\Pr_{a}\left[\left|\widehat{a}(T) - \ast\right| \ge \epsilon\right] \ge \delta \; \text{ for any } T \text{ such that }\; T \Gramm_T \le \frac{\log (1/2\delta)}{8\epsilon^2}. 
\end{eqnarray*}
\end{thm}
Theorem~\ref{thm:one_d_thm} is proven in Section~\ref{sec:proof_one_d} below, and Theorem~\ref{thm:info_lb_1d} is proven in Section~\ref{sec:lb_proofs}.

\section{Proof of Theorem~\ref{thm:one_d_thm}\label{sec:proof_one_d}}
To prove Theorem~\ref{thm:one_d_thm}, we write the error $E = \hat a - a = \frac{\sum_{t = 0}^{T - 1}x_t \noise_{t}}{\sum_{t = 0}^{T - 1}x_t^2}$. Since are interested in upper bounding the probability that $|E| > \epsilon$  it suffices to to show that the following two probabilities are small:
\begin{align*}
\Pr\left(\epsilon \sum_{t = 0}^{T - 1} x_t^2 - \sum_{t = 0}^{T - 1} x_t \noise_t < 0\right) \quad \text{and} \quad \Pr\left(\epsilon \sum_{t = 0}^{T - 1} x_t^2 + \sum_{t 0}^{T - 1} x_t \noise_t < 0\right).
\end{align*}
These probabilities are upper bounded by a standard Chernoff bound
\begin{align}
\label{eq:mgf_1d}
\Pr\left(\epsilon \sum_{t = 0}^{T - 1} x_t^2 \pm \sum_{t = 0}^{T - 1} x_t \noise_t < 0\right) \leq \inf_{\lambda \leq 0}\Exp \exp\left(\lambda \epsilon \sum_{t = 0}^{T - 1} x_t^2 \pm \lambda \sum_{t = 0}^{T - 1} x_t \noise_t\right).
\end{align}
We will apply this equation with $\lambda = -\epsilon$, controlling its magnitude with following lemma, proved in Section~\ref{sec:lem_one_step_1d} below:
\begin{lemma}
\label{lem:one_step_1d_mgf}
Let $a$, $\nu$, $\mu$, and $x$ be real numbers with $\nu < 1$ and let $\noise \sim \calN(0, 1)$. Then

$\Exp_\noise \exp \left(\frac{\nu}{2} (ax + \noise)^2 + \mu x \noise\right) = \frac{\exp\left(x^2\frac{\nu a^2 + 2\nu a \mu + \mu^2}{2(1- \nu)}\right)}{\sqrt{1 - \nu}}$.
\end{lemma}
With this lemma in hand, we can construct a recursive sequence which upper bounds $|a-\widehat{a}|$ with high probability:
\begin{prop}\label{prop:1d_mgf_bound}
Let $a$ be a real number and for $\alpha \in \R_{+}$ and $\epsilon \in (0,1)$ define recursively the sequence $\rho_t$ by $\rho_{T - 1} = 1$ and
\begin{align*}
\rho_{t} =
\begin{cases}
1 + r\rho_{t+ 1} & \rho_{t + 1} \leq \alpha/\epsilon^2,\\
\alpha/\epsilon^2 & \rho_{t + 1} > \alpha / \epsilon^2.
\end{cases}
\quad \text{ where } r = \frac{(|a| - \epsilon)^2}{1 + \alpha}.
\end{align*}
With this notation, $\Pr\left(|\widehat a - a| \leq \epsilon \right) \leq 2\exp \left(-\frac{\epsilon^2}{2(1 + \alpha)}\sum_{t = 1}^{T - 1}\rho_t\right).$
\end{prop}
\begin{proof}
The proof of this result is similar to the proof of the Azuma-Hoeffding inequality. It requires upper-bounding the MGF introduced in \eqref{eq:mgf_1d} by inductively applying the tower property of conditional expectation. We detail the proof in Section~\ref{proof:1d_mgf}.
\end{proof}

The proof of Theorem~\ref{thm:one_d_thm} is concluded in Section~\ref{sec:one_d_thm}, where we upper bound the sum $\sum_{t = 1}^{T - 1}\rho_t$, and solve for $T$.

\subsection{Proof of Lemma~\ref{lem:one_step_1d_mgf}\label{sec:lem_one_step_1d}}
\begin{align*}
&\Exp_\noise \exp \left(\frac{\nu}{2} (ax + \noise)^2 + \mu x \noise\right) = e^{\frac{\nu}{2} a^2 x^2} \Exp_\noise e^{\frac{\nu}{2} \noise^2 + \noise x(\nu a + \mu)}\\
&= \frac{e^{\frac{\nu}{2} a^2 x^2} }{\sqrt{2\pi}}\int_{-\infty}^\infty e^{\frac{\nu - 1}{2} \noise^2 + \noise x(\nu a + \mu)} d\noise = e^{\frac{\nu}{2} a^2 x^2} \frac{e^{x^2\frac{(\nu a + \mu)^2}{2(1 - \nu)}}}{\sqrt{1 - \nu}} = \frac{\exp\left(x^2\frac{\nu a^2 + 2\nu a \mu + \mu^2}{2(1 - \nu)}\right)}{\sqrt{1 - \nu}}.
\end{align*}

\subsection{Proof of Theorem~\ref{thm:one_d_thm}\label{sec:one_d_thm}}
Once again we let $a \geq 0$ for simplicity and recall from Proposition~\ref{prop:1d_mgf_bound} that we denote $r = (a - \epsilon)^2/(1 + \alpha)$. We study the case $a \leq 1$ first. Let us consider the sequence $\rho_t$ introduced in Proposition~\ref{prop:1d_mgf_bound} with $\alpha = 2\epsilon$ and note that
\begin{align*}
1 + r + \ldots + r^{t} \leq 1 + (1 + 2\epsilon)^{-1} + \ldots + (1 + 2\epsilon)^{-t} \leq \frac{1}{1 - (1 + 2\epsilon)^{-1}} \leq \frac{2}{\epsilon},
\end{align*}
which shows that for all $t$ we have $\rho_{T - 1 - t} = 1 + r + \ldots + r^t$ and hence $\sum_{t = 1}^{T - 1}\rho_t = \sum_{t = 1}^{T - 1} \frac{1 - r^{t}}{1 - r} = \frac{T}{1 - r} - \frac{\sum_{t = 0}^{T - 1} r^t}{1 - r}.$
Since $T/2 \geq 1 + r + r^2 + \ldots + r^{T - 1}$ when $T \geq 6/\epsilon$, we obtain that $\sum_{T = 1}^{T - 1} \rho_t \geq \frac{T}{2(1 - r)}, $, which, together with Proposition~\ref{prop:1d_mgf_bound}, it implies that
\begin{align*}
\Pr(|\widehat a -  a| \leq \epsilon) \leq 2\exp\left(-\frac{\epsilon^2T}{4(1 + 2\epsilon)(1 - r)}\right) = 2\exp\left(-\frac{\epsilon^2T}{4(1 + 2\epsilon - (a - \epsilon)^2)}\right).
\end{align*}
The first part of the corollary follows immediately.

We turn to the case $|a| > 1 + \epsilon$. Once again we assume $a > 0$ for simplicity. Recall that we have the freedom to choose any $\alpha \in \R_{+}$ for defining the sequence $\rho_t$.
Since $a > 1 + \epsilon$, if we choose $\alpha < (a - \epsilon)^2 - 1$ we guarantee that $r > 1$. To satisfy this inequality we choose $\alpha = ((a - \epsilon)^2 - 1)/2$.
Then, with this choice of $\alpha$, the sequence $\rho_t$ grows exponentially to $\alpha/\epsilon^2$. More precisely, by construction, since
\begin{align*}
\left[\frac{(a - \epsilon)^2}{1 + \alpha}\right]^{T - 2} = \left[\frac{2(a - \epsilon)^2}{1 + (a - \epsilon)^2}\right]^{T - 2} \geq (a - \epsilon)^{T - 2},
\end{align*}
$\rho_{1}$ is guaranteed to be equal to $\alpha/\epsilon^2$ as long as $(a - \epsilon)^{T - 2} \geq \alpha/\epsilon^2$. This last inequality holds when $T \geq \frac{\log\left(\frac{(a - \epsilon)^2 - 1}{2\epsilon^2}\right)}{\log(a - \epsilon)} + 2.$ In particular, if we choose $T$ to be at least double the right-hand side of the previous expression, then at least half of the terms $\rho_t$ are equal to $\alpha/\epsilon^2$, implying
\begin{align*}
\Pr(|\widehat a -  a| \leq \epsilon) \leq 2\exp\left(-\frac{\alpha T}{4(1 + \alpha)}\right).
\end{align*}
The conclusion now follows easily.

\subsection{Proof of Proposition~\ref{prop:1d_mgf_bound}}
\label{proof:1d_mgf}
We restrict ourselves to the case $a \geq 0$ (the case $a < 0$ can be analyzed analogously), and hence $r = (a - \epsilon)^2/(1 + \alpha)$. We  upper bound the MGF~\eqref{eq:mgf_1d} when $\lambda = -\epsilon$. Note that
\begin{align*}
\Exp \exp\left(- \epsilon^2 \sum_{t = 0}^{T - 1} x_t^2 \pm \epsilon \sum_{t = 0}^{T - 1} x_t \noise_t\right) &= \Exp \left[e^{ - \epsilon^2 \sum_{t = 0}^{T - 1} x_t^2 \pm \epsilon \sum_{t = 0}^{T - 2} x_t \noise_t}\Exp_{\noise_{T - 1}} \left[e^{\pm\epsilon x_{T - 1} \noise_{T - 1}}\left|\calF_{T - 1}\right.\right]\right]\\
&= \Exp \left[e^{ - \epsilon^2 \sum_{t = 0}^{T - 2} x_t^2 \pm \epsilon \sum_{t = 0}^{T - 3} x_t \noise_t}\Exp\left[ e^{-\frac{\epsilon^2}{2} x_{T - 1}^2 \pm \epsilon x_{T - 2} \noise_{T - 2}}\left| \calF_{T - 2}\right.\right]\right] \:.
\end{align*}

Then, from Lemma~\ref{lem:one_step_1d_mgf} we can upper bound the MGF by induction on $k$ by
\begin{align*}
\Exp \left[e^{ - \epsilon^2 \sum_{t = 0}^{T - k - 1} x_t^2 - \epsilon \sum_{t = 0}^{T - k - 2} x_t \noise_t}\Exp\left[ e^{-\frac{\epsilon^2\beta_{T - k}}{2} x_{T - k}^2 - \epsilon x_{T - k - 1} \noise_{T - k - 1}}\left| \calF_{T - k - 1}\right.\right]\right] \prod_{j = T - k + 1}^{T - 1} (1 + \epsilon^2 \beta_j)^{-1/2},
\end{align*}
where $\beta_t$ is any positive sequence such that $\beta_{T - 1} = 1$ and for $1 \leq t < T - 1$ it satisfies $\beta_t \leq 1 + \frac{\beta_{t +1}(a - \epsilon)^2}{1 + \epsilon^2 \beta_{t + 1}}$. It is straightforward to check that the sequence $\rho_t$ defined in the proposition statement above satisfies this
recursive inequality for any $\alpha \in (0,1)$. Therefore, we obtain the upper bound
\begin{align*}
\Exp \exp\left(- \epsilon^2 \sum_{t = 0}^{T - 1} x_t^2 - \epsilon \sum_{t = 0}^{T - 1} x_t \noise_t\right) &\leq \prod_{t = 1}^{T - 1}(1 + \epsilon^2 \rho_t)^{-1/2} = \exp\left(\sum_{t = 1}^{T - 1}-\frac{1}{2}\log (1 + \epsilon^2 \rho_t)\right) \\
&\leq \exp\left(\sum_{t = 1}^{T - 1}-\frac{\epsilon^2 \rho_t}{2(1 + \epsilon^2 \rho_t)}\right) \leq \exp\left(-\frac{\epsilon^2}{2(1 + \alpha)} \sum_{t = 1}^{T - 1}\rho_t\right) \:.
\end{align*}

%%% Local Variables:
%%% mode: latex
%%% TeX-master: t
%%% End:

%!TEX root = LWM.tex
\section{Proof of Theorem~\ref{main_thm} \label{app:main_thm}}
In this section, we conclude the technical aspects of the proof of Theorem~\ref{main_thm}. Recall the definition of the events
\begin{align*}
  \calE_1 := \left\{ \opnorm{\matU^{\top} \Noise} \geq K \right\} \:, \quad \calE_2 := \left\{\matX^\top\matX \succeq \frac{k\lfloor T/k\rfloor p^2 \Gamsb }{16}\right\} \:, \quad \calE_3 := \left\{ \matX^\top\matX \npreceq \Gambar \right\} \:.
\end{align*}
As we recall from Section~\ref{sec:mn_thm_proof}, if we can show that $\Pr[ \calE_2^c \cap \calE_3^c]$ and $\Pr\left[ \calE_1 \cap \calE_2 \cap \calE_3^c \right]$ where bounded above by $\delta$, we have
\begin{eqnarray}
 \Pr[\opnorm{\matX^{\dagger} \Noise} \geq \frac{4K}{p\sqrt{k\floor{T/k}\lambda_{\min}(\Gamsb) }} ] \le 3\delta
\end{eqnarray}
Observe that our condition on $k$ implies that necessarily $k \le T/10~$, so that $k\floor{T/k} \ge T - k \ge 9/10T$. Hence, we will have established
\begin{eqnarray}
 \Pr[\opnorm{\matX^{\dagger} \Noise} \geq \frac{10}{9} \cdot \frac{4K}{p\sqrt{T\lambda_{\min}(\Gamsb) }} ] \le 3\delta~,
\end{eqnarray}
and substiting in 
\begin{eqnarray}
K =  20\sigma \sqrt{n \log  + d\log \frac{10}{p} + \log \det(\Gambar\Gamsb^{-1}) + \log(1/\delta) }~.
\end{eqnarray}
proves the theorem. 

\subsection{Bounding $\Pr[ \calE_2^c \cap \calE_3^c]$ \label{app:main_sigm_min}}

Substituting the definitions of $\calE_2$ and $\calE_3$,
\begin{eqnarray*}
\Pr[ \calE_2^c \cap \calE_3^c]= \Pr\left[\left\{ \matX^\top\matX \nsucceq \frac{k\lfloor T/k\rfloor p^2 \nu^2 }{16}\right\} \cap \{\matX^\top\matX \preceq \Gamax\}\right]~.
\end{eqnarray*}

Proposition~\ref{prop:Small_Ball} and Equation~\eqref{eq:main_thm_condition} imply
\begin{eqnarray}\label{eq:one_w}
\forall \direc \in \R^{d},~\Pr\left[ \|\matX \direc\|^2 \le \frac{k\lfloor T/k\rfloor p^2 \direc^\top \Gamsb \direc }{8}\right] \le \exp\left( - \frac{\floor{T/k}p^2}{8}\right) %\le \frac{\delta}{3} \cdot \left(1+\frac{128 \itlamax}{T p^2 \nu^2}\right)^{-d}.
\end{eqnarray}
We apply Lemma~\ref{lem:eig_Packing_Lem} with $Q =  \matX$, with $\Gamax \leftarrow T\Gambar$, $\Gamin \leftarrow k \floor{T/k} p^2 \Gamsb/8$, and $\calT$ a net $1/4$-net of $\sphereGamin$ in the norm $\normGamax$. We shall use the following estimate of $|\calT|$:
\begin{lem}\label{lem:covering_number} Let $0 \prec \Gamin \preceq \Gamax$, and let $\calT$ be a minimal $\epsilon \le 1/2$-net of $\sphereGamin$ in the norm $\normGamax$. Then, $\log |\calT| \le d\log(1+\frac{2}{\epsilon}) + \log \det (\Gamax \Gamin^{-1})$. 
\end{lem}
\begin{proof} The covering number of $\sphereGamin$ in the norm $\normGamax$ is the same as the covering number of the shell of the ellipsoid $E := \{\direc: \direc^{\top} \Gamin^{-1/2}\Gamax \Gamin^{-1/2}\direc\}$ in the norm $\|(\cdot)\|_2$. Consider a maximal $\epsilon$-separated set of $\mathrm{bd}(E)$. Letting $\calB$ denote the unit ball in $\R^d$, a standard volumetric argument shows 
\begin{eqnarray*}
|\calT| &\le& \frac{\vol(\frac{2}{\epsilon}\calB + E)}{\vol(\frac{2}{\epsilon})} \\
&\overset{(i)}{\le}& \frac{\vol((\frac{2}{\epsilon}+1)E)}{\vol(\calB)} = \left(\frac{2}{\epsilon}+1\right)^d\frac{\vol(E)}{\vol(\calB)} \\
&=& \left(\frac{2}{\epsilon}+1\right)^d \det(\Gamin^{-1/2}\Gamax \Gamin^{-1/2})\\
&=& \left(\frac{2}{\epsilon}+1\right)^d \det(\Gamax \Gamin^{-1})\\
\end{eqnarray*}
where $(i)$ uses the inclusion $\calE \subset \calE$. 
\end{proof}
For $\epsilon = 1/4$ and our choise of $\Gamin$, we have
\begin{eqnarray}
\log |\calT| &=& d\log(1+\frac{2}{\epsilon}) + \log \det (\Gamax \Gamin^{-1})\\
&=& d\log(9) + d \log (8T/(\floor{T/k}k) p^2) +  \log \det (\Gambar \Gamsb^{-1})\nonumber\\
&\le& d\log(9) + d \log (72/ 10p^2) +  \log \det (\Gambar \Gamsb^{-1})~\quad \text{ since } k \le T/2\nonumber\\
&\le& 2d\log(10/ p) +  \log \det (\Gambar \Gamsb^{-1})~ \label{eq:calt_bound}.
\end{eqnarray}
Hence, we conclude that
\begin{eqnarray*}
&&\Pr[ \calE_2^c \cap \calE_3^c]= \Pr\left[\left\{ \matX^\top\matX \nsucceq \frac{k\lfloor T/k\rfloor p^2 \nu^2 }{16}\right\} \cap \{\matX^\top\matX \preceq \Gamax\}\right] \\
&\le& \Pr\left[ \left\{\exists \direc \in \calT : \|\matX \direc\|^2 < \frac{k\floor{T/k}T p^2 \direc^\top\Gamsb \direc }{8}\right\} \cap \{\matX^\top\matX \preceq \Gamax\}\right] \\
&\le&  \exp\left( - \frac{\floor{T/k}p^2}{8}+ 2d\log(10/ p) +  \log \det (\Gambar \Gamsb^{-1}) \right)\\
&\le&  \exp\left( - \frac{Tp^2}{10k} + 2d\log(10/ p) +  \log \det (\Gambar \Gamsb^{-1})\right)
\end{eqnarray*}

\subsection{Bounding $\Pr\left[ \calE_1 \cap \calE_2 \cap \calE_3^c \right]\label{app:main_noise}$}
We shall need the following discretization lemma, proved in Appendix~\ref{sec:proof_lem_lin_pack}:
\begin{lem}\label{lem:Lin_packing} Let $Q \in \R^{n\times m}$ have full column rank,
$q \in \R^n$, let $0 \prec \Gamin \preceq Q^TQ \preceq \Gamax$, and let $\calT$ be a $1/4$-net of $\sphereGamin$ in the norm $\normGamax$.
Then, 
\begin{eqnarray}\sup_{\direc \in \calS^{m-1}} \frac{ \langle Q \direc,  q \rangle}{\|Q \direc\|} \le 2\max_{\direc \in \calT}  \frac{ \langle Q \direc, q \rangle}{\|Q \direc\|}
\end{eqnarray}
\end{lem}

To control the size of $\opnorm{\matU^\top \Noise}$ we use a variational formulation of the operator norm and two coverings. Lett $\R_*^d := \R^d - \{0\}$. Note that if $\calT_1$ is a $1/2$-net of $\calS^{n-1}$ (over the $v)$, then, 
\begin{eqnarray*}
\opnorm{\matU^\top \Noise} &\le&  \sup_{v \in \calS^{n-1}, \direc \in \R_*^{d}} \frac{\direc^{\top}\matU^{\top} \Noise v}{\|\direc\|} \\
\ &\le&  2\max_{v \in \calT_1}\left(\sup_{\direc \in \R_*^{d}} \frac{\direc^{\top}\matU^{\top} \Noise v }{\|\direc\|}  \right)\quad \text{ since } \calT_1  \text{ is a } 1/2 \text{-net} \\
&=&  2\max_{v \in \calT_1} \left(\sup_{\direc \in \R_*^{d}} \frac{\direc^{\top}\matV\matSig\matU^{\top} \Noise v }{\| \matSig \matV^\top \direc\|} \right) \quad \text{ since }\matSig\matV^{\top} \text{ is full rank on } \calE_2\\
&=& 2 \max_{v \in \calT_1}\left(\sup_{\direc \in \R_*^{d}} \frac{\direc^{\top}\matX^{\top} \Noise v }{\|\matX \direc\|}\right) = 2 \max_{v \in \calT_1}\left(\sup_{\direc \in \calS^{d-1}} \frac{\direc^{\top}\matX^{\top} \Noise v }{\|\matX \direc\|}\right),
\end{eqnarray*}
where the second-to-last equation uses $\matX = \matU \matSig \matV^{\top}$ and $\|\matSig \matV^\top \direc\| = \|\matU \matSig \matV^\top \direc\| = \|\matX \direc\|$. Define $\calE_g := \calE_2 \cap \calE_3^c$. We see then that on $\calE_g$, we have 
\begin{eqnarray*}
\frac{k\lfloor T/k\rfloor p^2 \Gamsb^2 }{16} \preceq \matX^\top\matX \preceq T\Gamax~.
\end{eqnarray*}
We now apply Lemma~\ref{lem:Lin_packing} with $Q =  \matX$, with 
\begin{eqnarray}\label{Gam:def}
\Gamax \leftarrow T\Gambar & \text{and } & \Gamin \leftarrow k \floor{T/k} p^2 \Gamsb/16
\end{eqnarray}
and $\calT_2$ a net $1/4$-net of $\sphereGamin$ in the norm $\normGamax$. This yields
\begin{eqnarray*}
&&\Pr\left[\left\{\opnorm{\matU^\top \Noise} > K\right\} \cap \calE_g\right] \\
&\le& \Pr\left[\left\{\max_{v \in \calT_1} \sup_{\direc \in \calS^{d-1}}\frac{\direc^{\top}\matX^{\top} \Noise v }{\|\matX \direc\|} > K/2\right\} \cap \calE_g\right] \\
&\le& \Pr\left[\left\{\max_{v \in \calT_1} \max_{\direc \in \calT_2}\frac{\direc^{\top}\matX^{\top} \Noise v }{\|\matX \direc\|} > K/4\right\} \cap \calE_g\right] \\
&\le& |\calT_1||\calT_2|\sup_{v \in \calT_1,\direc \in \calT_2} \Pr\left[\left\{\frac{\direc^{\top}\matX^{\top} \Noise v }{\|\matX \direc\|} > K/4\right\} \cap \calE_g\right] \:.
\end{eqnarray*}
To obtain a pointwise bound on $\Pr\left[\left\{\frac{\direc^{\top}\matX^{\top} \Noise v }{\|\matX \direc\|} > K/4\right\} \cap \calE_g\right]$ we use Lemma~\ref{lem:martingale_lem}, with $Z_t = \langle X_t, \direc \rangle$, 
$\noiseb_t = \langle \noise_t, v \rangle$, and the bounds $\beta_- = \direc^\top \Gamin \direc$, and $\beta_+ = \direc^\top \Gamax \direc$, $\Gamin$, $\Gamax$ are as in Equation~\eqref{Gam:def}. We can then bound
\begin{eqnarray*}
&&|\calT_1||\calT_2|\sup_{v \in \calS^{d-1},\direc \in \calS^{d-1}}\Pr\left[\left\{\frac{\direc^{\top}\matX^{\top} \Noise v }{\|\matX \direc\|} > K/4\right\} \cap \calE_g\right] \\
&\le& |\calT_1||\calT_2|\sup_{v \in \calS^{d-1},\direc \in \calS^{d-1}}\Pr\left[\left\{\frac{\direc^\top \matX^\top \Noise v}{\|\matX \direc\|} > K/4\right\} \cap \left\{ \|\matX \direc\|^2 \in \left[\frac{k\floor{T/k} \nu^2 p^2}{16},\itlamax \right]\right\}\right] \\
&\le& |\calT_1||\calT_2| \log_+ \ceil{\frac{\beta_+}{\beta_-}} \exp(-K^2 / 96 \sigma^2 )\\
&\overset{(i)}{\le}& \exp( n \log 5 + d\log 9 + \log \det(\frac{32}{p^2}\Gamax\Gamin^{-1}))  \log_+ \ceil{\frac{\beta_+}{\beta_-}} \exp(-K^2 / 96\sigma^2 ) \:.
\end{eqnarray*}
where $(i)$ uses the standard metric entropy bound for the sphere(see, e.g.~\cite{vershynin11}), and an analogous computation to~\eqref{eq:calt_bound}.
Now since $\direc \ne 0$, bound $\log \ceil{x} \le \log 1 + x \le x$, and computing as in~\eqref{eq:calt_bound} yields the bound
\begin{eqnarray*}
\log\ceil{\frac{\beta_+}{\beta_-}} ~\le~\frac{\beta_+}{\beta_-} &\le&  \sup_{\direc \in \R^d - \{0\}} \frac{\direc^\top \Gamax \direc}{\direc^\top \Gamin \direc} \\
&=& \|\Gamin^{-1/2}\Gamax\Gamin^{-1/2}\| = \lambda_{\max}(\Gamax \Gamin^{-1})\\
&=& \lambda_{\max}(\frac{32}{p^2}\Gambar\Gamsb^{-1}))\\
&\le& \exp( \log \lambda_{\max}(\frac{32}{p^2}\Gambar\Gamsb^{-1}))~)\\
&\overset{(i)}{=}& \exp( \log \det(\frac{32}{p^2}\Gambar\Gamsb^{-1})))~.
\end{eqnarray*}
where $(i)$ uses the fact that $\frac{32}{p^2}\Gambar\Gamsb^{-1}$ has the same eigenvalues as $\frac{32}{p^2}\Gamsb^{-1/2}\Gambar\Gamsb^{-1/2} \succeq \frac{32}{p^2} \succeq I$.  All together, we have
\begin{eqnarray*}
&&\Pr\left[\left\{\opnorm{\matU^\top \Noise} > K\right\} \cap \calE_g\right]\\ &\le& \exp( n \log 5 + d\log 9 + 2\log \det(\frac{32}{p^2}\Gambar\Gamsb^{-1}))\exp(-K^2 / 96\sigma^2 ) \\
&\le& \exp( n \log 5 + 2d\log \frac{96}{p^2} + 2\log \det(\Gambar\Gamsb^{-1}))\exp(-K^2 / 96\sigma^2 ) \\
&\le& \exp( 4(n + d\log \frac{10}{p} + \log \det(\Gambar\Gamsb^{-1})))\exp(- (K/10\sigma)^2) 
\end{eqnarray*}
Hence, we guarantee that $\Pr\left[ \calE_1 \cap \calE_2 \cap \calE_3^c \right] \leq \delta $ if we choose
\begin{eqnarray*}
K = 20\sigma \sqrt{n \log  + d\log \frac{10}{p} + \log \det(\Gambar\Gamsb^{-1}) + \log(1/\delta) }~.
\end{eqnarray*}

%!TEX root = LWM.tex
\section{Proof of Technical Results\label{sec:proof_technical}}
 \subsection{Proof of Proposition~\ref{prop:Small_Ball}\label{sec:proof_small_ball}}
To exploit the $(k, \nu, p)$ block martingale small-ball condition we partition the sequence of random variables $Z_1, Z_2, \ldots, Z_T$ into
 $\lfloor T/k \rfloor$ blocks of size $k$ (we discard the remainder terms). For simplicity we denote $S = \lfloor T/k\rfloor$. We consider the random variables
\begin{align*}
\block_{j} &= \I\left(\sum_{i=1}^k Z_{jk+i}^2 \ge  \frac{\nu^2 p k}{2} \right)\quad \text{for } 0 \leq j \leq S - 1.%\quad  W_{S} &= \I\left(\sum_{i=1}^{T - kS} Z_{kS + i}^2 \ge k \nu^2/2 \right).
\end{align*}
Given this notation, we can use the Chernoff bound to obtain
\begin{align}\label{eq:chernoff_martingale}
  \Pr\left[\sum_{i=1}^T Z_i^2 \le \frac{\nu^2p^2}{8} k S \right] &\leq \Pr\left[\sum_{j=0}^{S - 1} \block_{j} \le \frac{p}{4}S \right]\leq \inf_{\lambda \leq 0} e^{-\frac{p}{4}S}\Exp e^{\lambda \sum_{j = 0}^{S - 1}\block_j}.
\end{align}
The first inequality above uses the trivial inequality
$\sum_{i=1}^{k} Z_{jk+i}^2 \geq \frac{\nu^2 p k}{2} \block_j$.

For upper bounding the MGF on the right hand side we will use the tower property with respect to the filtration $\calF_{jk}$ for $j$ from $S-1$ to $0$. Before turning to that computation it is valuable to lower bound the conditional expectations $\Exp \left[\block_{j} |\calF_{jk}\right]$:
\begin{align*}
  \Exp \left[\block_{j} |\calF_{jk}\right] &= \Pr\left[\sum_{i = 1}^{k} Z_{jk + 1}^2 \geq \frac{\nu^2 p k}{2}\left|\calF_{jk} \right.\right] \stackrel{(a)}{\geq} \Pr\left[\frac{1}{k}\sum_{i = 1}^k \I\left(|Z_{jk + 1}| \geq \nu\right) \geq \frac{p}{2}\left| \calF_{jk}\right.\right] \\
  &\stackrel{(b)}{\geq} \frac{(p/2)}{1-(p/2)} \geq \frac{p}{2},
\end{align*}
where (a)
uses the trivial inequality $\frac{1}{\nu^2} Z_{jk+1}^2 \geq \I\left(|Z_{jk + 1}| \geq \nu\right)$,
and (b) uses inequality follows from the $(k, \nu, p)$-BMSB condition and the following claim is straightforward.

\emph{Claim}: Let $Z$ be a random variable supported in $[0,1]$ almost surely such that $\Exp[Z] \geq p$ for some $p \in (0,1)$. Then, for all $t \in [0,p]$, $\Pr[Z \geq t] \geq \frac{p - t}{1 - t}$.

From this lower bound on $\Exp \left[\block_{j} |\calF_{jk}\right]$, using $\lambda \leq 0$, we get
\begin{align*}
\Exp \left[e^{\lambda \block_j}\left| \calF_{jk}\right. \right] &= e^\lambda \Pr\left[ \block_j = 1| \calF_{jk}\right] + \Pr\left[\block_j = 0\right] = (e^\lambda - 1)\Exp \left[ \block_j |\calF_{jk}\right] + 1\\
&\leq (e^\lambda - 1)\frac{p}{2} + 1.
\end{align*}

Now, by inductively conditioning on $\calF_{jk}$, we can upper bound
\begin{align*}
\Exp e^{\lambda \sum_{j = 0}^{S - 1}\block_j} &= \Exp \left[e^{\lambda \sum_{j = 0}^{S - 2}\block_j}\Exp \left[ e^{\lambda \block_{S - 1}}    \left| \calF_{(S - 1)k}\right.\right]\right] \leq \left((e^\lambda - 1)\frac{p}{2} + 1\right)\Exp \left[e^{\lambda \sum_{j = 0}^{S - 2}\block_j}\right]\\
&\leq \left((e^\lambda - 1)\frac{p}{2} + 1\right)^{S}.
\end{align*}

We now plug in this upper bound in Equation~\ref{eq:chernoff_martingale} and optimize for $\lambda$. From the first order optimality condition it is easy to see that the optimal choice of $\lambda$ is
\begin{align*}
\lambda_\star = \log \left(\frac{1 - p/2}{2 - p/2}\right).
\end{align*}
Plugging in $\lambda_\star$ back in Equation~\ref{eq:chernoff_martingale}, after some elementary calculus and algebraic manipulations, we find the desired conclusion.

\subsection{Proof of Martingale Concentration (Lemma~\ref{lem:martingale_lem})\label{proof:lem_martingale}}
	For ease of notation we denote $S_t = \sum_{s = 1}^t Z_s \noiseb_s$ and $R_t = \sum_{s = 1}^t Z_s^2$.

	\emph{(a) } Using a Chernoff argument,  we have

		\begin{eqnarray*}
		\Pr\left[ \left\{S_T \ge \alpha\right\} \cap \left\{ R_T \le \beta\right\}\right] &=& \inf_{\lambda > 0}\Pr\left[ \{e^{\lambda S_T} \ge e^{\lambda \alpha}\} \cap \{R_T \le \beta\}\right]\\
	&=& \inf_{\lambda > 0}\Pr\left[e^{\lambda S_T}\I(R_T \le \beta) \geq e^{\lambda \alpha} \right] \\
		&\le& \inf_{\lambda > 0}e^{-\lambda \alpha} \Exp[e^{\lambda S_T} \I(R_T \le \beta)]\\
		&=& \inf_{\lambda > 0}e^{-\lambda \alpha }\cdot e^{\lambda^2 \sigma^2\beta/2}\Exp[e^{\lambda S_T -\lambda^2 \sigma^2 \beta / 2} \I(R_T \le \beta)]\\
		&\le& \inf_{\lambda > 0} e^{-\lambda \alpha }\cdot e^{\lambda^2 \sigma^2\beta/2}\Exp[e^{\lambda S_T -\lambda^2 \sigma^2R_T / 2}].
		\end{eqnarray*}
		Now, we claim that $\Exp[e^{\lambda S_T -\lambda^2 \sigma^2 R_T/2}] \le 1$. Indeed, by the tower rule
	  and the assumption that $\noiseb_t | \calF_t$ is a zero mean $\sigma$-sub-Gaussian r.v.,
	  we have
		\begin{eqnarray}
		\Exp[\exp(\lambda S_T - \lambda^2 \sigma^2 R_T/2) ] &=& \Exp[\Exp[\exp(\lambda S_T - \lambda^2 \sigma^2 R_T/2)  \big{|} \calF_{T}]]\nonumber\\
		&\le& \Exp[\exp(\lambda S_{T-1} - \lambda^2 \sigma^2 R_{T-1}/2)\Exp[e^{\lambda Z_T \noiseb_T - \lambda^2 \sigma^2 Z_T^2 / 2} | \calF_{T}]  ] ]\nonumber\\
	  &\le& \Exp[\exp(\lambda S_{T-1} - \lambda^2 \sigma^2 R_{T-1}/2)]\nonumber\\
	  &\vdots& \nonumber\\
		&\le&\Exp[\exp(\lambda S_{1} - \lambda^2 \sigma^2 R_{1}/2)] \le 1.\label{eq:last_line_exp}
	 	\end{eqnarray}
		Hence,
		\begin{eqnarray*}
		\Pr[ \{S_t \ge \alpha\} \cap \{R_T \le \beta\}] &\leq& \inf_{\lambda > 0}  e^{-\lambda \alpha }e^{\lambda^2\sigma^2 \beta/2} = e^{-\alpha^2/2\sigma^2\beta}.
	\end{eqnarray*}

	\emph{(b) }  Let $B := \log\ceil{\frac{\beta_+}{\beta_-}}$. Then
		\begin{eqnarray*}
		\Pr\left[\{S_T > \alpha\sqrt{R_T}\} \cap \{\beta_- \le R_T \le \beta_+\}\right] &\le& \Pr\left[\{S_T > \alpha\sqrt{R_T}\} \cap \left\{\beta_- \le R_T \le e^B\beta_- \right\}\right]\\
		&=& \sum_{i=0}^{B-1}\Pr[\{S_T > \alpha\sqrt{R_T}\} \cap \{e^i\beta_- \le R_T \le e^{i + 1}\beta_-\}]\\
		&\le& \sum_{i=0}^{B-1}\Pr[\{S_T > \alpha\sqrt{e^{i}\beta_-}\} \cap \{e^{i} \beta_- \le R_T \le e^{i+1}\beta_- \}]\\
		&\le& \sum_{i=0}^{B-1}\Pr[\{S_T > \alpha\sqrt{e^{i}\beta_-}\} \cap \{R_T \le e^{i+1}\beta_- \}]\\
	  &\stackrel{(i)}{\le}& \sum_{i=0}^{B-1}\exp\left(\frac{-\alpha^2 e^i\beta_-}{2e^{i + 1}\sigma^2\beta_-}\right)\\
		&=& B\exp\left(\frac{-\alpha^2}{2e \sigma^2}\right) \leq  \log\left\lceil\frac{\beta_+}{\beta_-} \right \rceil \exp\left(\frac{-\alpha^2}{ 6\sigma^2}\right)~.
		\end{eqnarray*}
		Above, (i) follows from part (a) of the claim.
\subsection{Proofs of Covering Results}
\subsubsection{Proof of Lemma~\ref{lem:eig_Packing_Lem}\label{sec:proof_lem_eig_pack}} Consider the transformed matrix $\Gamin^{-1/2}Q\Gamin^{-1/2}$. It suffices to show that under the assumptions of Lemma~\ref{lem:eig_Packing_Lem}, 
\begin{eqnarray}
\inf_{\direc\in \sphereGamin} \|Q\direc\| \ge 3/4~,
\end{eqnarray}
since then $Q^TQ \succeq (3/4)^2\Gamin \succeq \Gamin/2$. Now $\direc, v \in \R^{d}$, we can bound
\begin{eqnarray*}
\|Q\Gamax^{-1/2}(\direc-w)\| \le \|\Gamax^{1/2}(\direc-w)\| 
\end{eqnarray*}
since $Q^\top Q \preceq \Gamax$. In particular, if $\calT$ is a $1/2$-net of $\sphereGamin$ in $\normGamax$, then 
\begin{eqnarray*}
\inf_{\direc \in \sphereGamin} \|Q\direc\| \ge \inf_{\direc \in \calT} \|Q \direc\| - \frac{1}{4} \ge \frac{1}{4}~.
\end{eqnarray*} 
where the last step follows from the assumption that $\inf_{\direc\in \calT} \|Q\direc\| \ge 1$.

\subsubsection{Proof of Lemma~\ref{lem:Lin_packing}\label{sec:proof_lem_lin_pack}}
Define the map $\phi(\direc):= \frac{Q \direc}{\|Q \direc\|}$. We shall prove that for all $v,\direc \in \calS^{m-1}$, one has that
\begin{eqnarray}\label{eq:phi_lip}
\|\phi(\direc) - \phi(v)\| \le \frac{2\|Q(v-w)\|}{\|Qv\|}
\end{eqnarray}
Note observe that, if $0 \prec \Gamin \preceq Q^\top Q \preceq \Gamax$. Hence, each $\direc \in \sphereGamin$ can be written as $\Gamin^{-1/2}\direc'$ for $\direc' \in \calS^{d-1}$, we have that
\begin{eqnarray*}
\inf_{w \in \sphereGamin} \|Qw\|^2 = \inf_{\direc' \in \sphereGamin} (\direc')^\top \Gamin^{-1/2}Q^\top Q \Gamin^{-1/2} \direc' \ge \|\direc'\|_2^2 = 1
\end{eqnarray*}
 and that 
 \begin{eqnarray}
 \|Q\direc\| = \direc^\top Q^\top Q \direc \le \direc^\top \Gamax \direc  = \|\Gamax^{1/2} \direc\|^2
 \end{eqnarray}
Thus, for all all $v,\direc \in \sphereGamin^{m-1}$,
\begin{eqnarray*}
\|\phi(\direc) - \phi(v)\| \le \frac{2\|Q(v-w)\|}{\|Qv\|} \le 2\|\Gamax^{1/2}(v-w)\| 
\end{eqnarray*}
Since $\Gamin^{-1/2} \succ 0$ is full rank, we have
\begin{eqnarray*}
	\sup_{\direc \in \calS^{d-1}} \frac{ \langle Q \direc,  q \rangle}{\|Q \direc\|} &=&  \sup_{\direc \in \sphereGamin^{d-1}} \frac{ \langle Q \direc,  q \rangle}{\|Q_{\min} \direc\|}\\
\end{eqnarray*}
and since $\calT$ is a $1/4$-net of $\calS^{d-1}$ in the norm $\normGamax$, \eqref{eq:phi_lip} implies the above is at most
\begin{eqnarray*}
2\sup_{\direc \in \calT} \frac{ \langle Q \direc,  q \rangle}{\|Q_{\min} \direc\|}~.
\end{eqnarray*}

	It remains to check~\eqref{eq:phi_lip}. 
	\begin{eqnarray*}
	\|\phi(v) - \phi(\direc)\| &=& \left|\frac{{Q}v}{\|{Q}v\|} - \frac{{Q}\direc}{\|{Q}\direc\|}\right|\\
	&\le&  + \|{Q}\direc\|\left|\frac{1}{\|Q v\|} - \frac{1}{\|Q \direc\| }\right|\\
	&\le& \frac{\|{Q}(v - \direc)\|}{\|{Q}v\|} +\frac{|\|{Q}\direc\|-\|Q v\||}{\|{Q}v\|}\\
	&\le& 2\frac{\|{Q}(\direc-v)\|}{\|{Q}v\|}~.\\
	\end{eqnarray*}

%\input{lin_sys_proofs}
%!TEX root = LWM.tex
\section{Lower Bounds}

%\subsection{Remarks on the Looseness of Theorem~\ref{thm:info_lb_d}\label{sec:Lower_Bound_Loose}}
\subsection{Proof of Information Theoretic Lower Bounds, Theorems~\ref{thm:info_lb_1d} and~\ref{thm:info_lb_d}\label{sec:lb_proofs}}
 In this section we prove Theorem~\ref{thm:info_lb_1d} and~\ref{thm:info_lb_d}. We shall the $\Pr_{A}^{(T)}$ denote the law of the iterates $X_{t+1} = AX_{t} + \noise_t$, where $\noise_t \sim \calN(0,I)$, for $t = 1,2,\dots,T$. We shall prove Theorems~\ref{thm:info_lb_1d} and~\ref{thm:info_lb_d} using Birge's Inequality, a bound which is qualitatively similar to Fano's inequality, but yields sharp high-probability lower bounds in low-dimensional settings.
\begin{lem}[Variant of Birge's Inequality]\label{lem:Birge} Let $\calE_0,\calE_1,\dots,\calE_N$ be disjoint events, $\Pr_0,\Pr_1,\dots,\Pr_N$ be probability laws, and let $\min_{i}\Pr(\calE_i^c) \le 1/2$. Then, for any $\delta \in (0,1/2)$,
\begin{eqnarray}
\sum_{i=1}^N \KL(\Pr_i,\Pr_0) \ge (1-2\delta) \log(N/2\delta)~.
\end{eqnarray}
In particular, fix an $\epsilon > 0$ and $\delta \in (0,1/2)$, and suppose that for a finite set $\calN \subset \R^{n \times n}$, all $A_1 \ne A_2 \in \calN$ satisfy $\|\calA_1 - \calA_2\|_{\op} \ge 2\epsilon$. Then if $\inf_{\ALS}\sup_{A \in \calN} \Pr_{A}[\|A - \ALS(T)\|_{\op} \ge \epsilon] \le \delta$, then $T$, $\delta$ and $|\calN|$ satisfy the following inequality for any $A_0 \in \calN$:
\begin{eqnarray}
\sup_{A \in \calN - \{A_0\}}\KL(\Pr_{A}^{(T)},\Pr_{A_0}^{(T)}) \ge (1-2\delta) \log\left(\frac{|\calN| - 1|}{2\delta}\right)~.
\end{eqnarray}
%  (1-\delta)\log (N\frac{1-\delta}{\delta}) + \delta \log( \frac{\delta}{1 - (\delta)/N})
\end{lem}
We prove Lemma~\ref{lem:Birge} from a standard statment of Birge's inequality from \citep[Theorem 4.20]{boucheron13}, in Section~\ref{BirgeSec}. Lemma~\ref{lem:Birge} relates the probability of error to the $\KL$-divergences between laws $\Pr_{A}^{(T)}$ in $2\epsilon$-separated set $\calN$. Thus, our first step will be to compute the term $\KL(\Pr_{A}^{(T)},\Pr_{A_0}^{(T)})$. This amounts to a straightforward computation, carried out in Section~\ref{sec:kl_comp}.
\begin{lem}\label{kl_comp} Let $O,O' \in \Od$. Then, $\KL(\Pr_{\rho}^{(T)},\Pr_{A}^{(T)}) =  \fronorm{\eig O - A}^2 \cdot \sum_{t=1}^T \gamma_{t}(\eig)$, where we recall $\gamma_{t}(\eig) = \sum_{s=0}^{t-1} |\eig|^{2s}$.
\end{lem}
We are now in a position to prove the lower bound in one-dimension:
\subsection{1-D Lower Bound: Proof of Theorem~\ref{thm:info_lb_1d}}
\begin{proof} Fix an $\eig \in \R$, and let $\eig' \in \{\eig - 2\epsilon, \eig + 2\epsilon\}$. Viewing $\eig,\eig'$ as matrices in $\R^{1 \times 1}$, we have Lemma~\ref{kl_comp}, implies $\KL(\Pr_{\eig}^{(T)},\Pr_{\eig'}^{(T)}) = 4\epsilon^2 \cdot \sum_{t=1}^T \gamma_{t}(\eig)$. Then, applying Lemma~\ref{lem:Birge} with $A_0 = \eig$ and $\calN = \{\eig,\eig'\}$, we have for if $\sup_{a \in \eig,\eig'}\Pr_a[|\widehat{a}(T) - a| < \epsilon ] \le \delta$, then, $\epsilon^2 \cdot \sum_{t=1}^T \gamma_{t}(\eig) \ge (1-2\delta) \log\left(\frac{1}{2\delta}\right)$. Hence, we need $T\gamma_T(\eig) \ge\sum_{t = 1}^T\gamma_t(\eig) \ge \frac{1}{4\epsilon}(1-2\delta) \log\left(\frac{1}{2\delta}\right)$.
\end{proof}
\subsection{$d$-Dimensional Lower Bound}
If we chose $\calN$ to be a $2\epsilon$-packing of the set $\eig\Od$, then Lemma~\ref{lem:Birge} and Lemma~\ref{kl_comp} imply that, for any estimate $\widehat{A}$ such that
\begin{eqnarray*}
&\sup_{ O \in \Od}\Pr_{\eig O}[\|\widehat{A}(T) - \eig O\|_{\op} \ge \epsilon] \ge \delta \text{ for any } \\
&T \text{ such that } (1-2\delta)\log \frac{|\calN|}{2\delta} \ge \left(\sum_{t=1}^T \gamma_{t}(\eig)\right) \cdot \max_{\eig O,\eig O' \in \calN}\fronorm{\eig O - \eig O'}^2~.
\end{eqnarray*}
In light of the above, our goal will be to construct a $2\epsilon$-packing $\calN$ such that $\inf_{\eig O,\eig O' \in \calN}\fronorm{\eig O - \eig O'}^2$ is as small as possible. This is achieved by the following proposition, which lifts a $1/2$-packing of the unit ball in $d-1$-dimensions to a packing $\calN_0$ of $\Od$, proved in Section~\ref{prop:packing_prop}:
\begin{prop}\label{prop:packing_prop} Fix an $\epsilon_0 \le 1/256$, and let $\calT$ be an $1/2$-packing of $B_{d-1}(1)$. Then, there exists a set $\overline{\calN} \subset \Od$ with $|\calN_0| = |\calT|$ and, for all $A_1 \ne A_2 \in \overline{\calN}$,
\begin{eqnarray}
\opnorm{A_1-A_2} \ge \epsilon_0 /4 \quad \text{ and } \quad \fronorm{A_1-A_2} \le 4\epsilon_0 \:.
\end{eqnarray}
\end{prop}
We now reparameterize the above proposition with $\epsilon_0 = \frac{8\epsilon}{\eig}$. Let $\calT$ be a maximal $1/2$-packing of $\calB_{d-1}(1)$; a standard fact shows that $|\calT| \ge 2^{d-1}$. Them, as long as $\epsilon \le \frac{\eig}{2048}$, $\calN = \eig \overline{\calN}$ is a $2\epsilon$-packing of the set $\eig \Od$, and for all $\eig O, \eig O' \in \calN$, $\fronorm{A_1-A_2} \le 32\epsilon$,
\begin{eqnarray*}
&\sup_{O \in \Od}\Pr_{\eig O}[\|\widehat{A}(T) - \eig O\|_{\op} \ge \epsilon] \ge \delta \text{ for any } \\
&T \text{ such that } (1-2\delta)\log \frac{2^d}{4\delta} \ge \left(\sum_{t=1}^T \gamma_t \right) \cdot (32\epsilon)^{2} \:.
\end{eqnarray*}
In particular, for $\delta \le 1/4$, we see that there exists a universal constant $c_0$ such that $(1-2\delta)\log \frac{2^d}{4\delta} \ge c_0(d + \log (1/\delta))$, and hence for $c = c_0/32^2$, we see that
\begin{eqnarray*}
\sup_{\eig O \in \Od}\Pr[\|\widehat{A}(T) - \eig O\|_{\op} \ge \epsilon] \ge \delta \text{ for any } T: \frac{c_0 \left( d + \log(1/\delta)\right)}{\epsilon^2} \ge \sum_{t=1}^T \gamma_{t}(\eig) \:.
\end{eqnarray*}
Bounding $\sum_{t=1}^T \gamma_{t}(\eig) \le T \gamma_{T}(\eig)$ concludes the proof.
\subsection{Proof of Proposition~\ref{prop:packing_prop}\label{sec:packing_sec}}
We now construction of the packing $\overline{\calN}$. If we define the set $\Skew(d) := \{X \in \R^{d\times d}: X^\top = - X\}$, and recall the matrix exponential $\exp(X) = \sum_{j=0}^\infty X^j/j!$, a well-known theorem in Lie Theory ensures that $\exp(\Skew(d)) \subset \Od$ (see, e.g.~\cite{knapp2016representation}). Moreover, $\exp$ is an approximate isometry (in both $\|\cdot\|_{\op}$ and $\fronorm{\cdot}$) from a small neighborhood of $0 \in \Skew(d)$ to a small neighborhood of the identity $I \in \Od$. Hence, our strategy will be to construct a packing in $\Skew(d)$, and then push it to $\Od$ under the $\exp$ mapping.

Formally, given $\epsilon \le 1/256$, and a $1/2$ pcking of  $B_{d-1}(1)$ $\calT$, define for $w \in \calT$ the matrix
\begin{eqnarray*}
M(w) := \epsilon \left(e_1(0,w)^\top + (0,w)e_1^\top\right) \in \Skew(d)~,
\end{eqnarray*}
where $e_1$ denotes the first canonical basis vector in $R^d$. Observe that $\|M(w)\|_F = \sqrt{2\|w\|^2} = \sqrt{2}\|w\|$ and, since the singular value of $M(w)$ are paired, we have $\|M(w)\|_{\op} = \|w\|$. Hence, for every $w_1 \ne w_2 \in \calB_{d-1}(1)$, we have
\begin{eqnarray*}
\|M(w_1-w_2)\|_{\op} &=& \epsilon \|w_1 - w_2\|_{2} \ge \epsilon /2 ~\text{and}~ \\
\fronorm{M(w_1-w_2)} &=&\sqrt{2}\epsilon \|w_1 - w_2\|_{2}  \le 2\sqrt{2\epsilon} \:.
\end{eqnarray*}
Now, we define our packing $\overline{\calN}$ formally as
\begin{eqnarray*}
\overline{\calN} := \{\exp(M(w)): w \in \calT\} \:.
\end{eqnarray*}
We now introduce the following lemma, proved in Section~\ref{sec:exp_map_proof}, which precisely describes the extent to which $\exp()$ is an isometry:
\begin{lem}\label{lem:exp_map} Let $\|\|$ be a sub-multiplicative norm (e.g., $\opnorm{\cdot}$ or $\fronorm{\cdot}$), and $X,Y \in \R^{d \times d}$. Then,
\begin{eqnarray}
\|\exp(X+Y) -\exp(X) - Y\| \le e^{2K} - 1 - 2K, \text{ where } K = \max\{\|X\|,\|Y\|\}~.
\end{eqnarray}
\end{lem}
We apply the above  with $Y = M(w_1) - M(w_2) = M(w_1 - w_2)$, and $X = M(w_2)$. Then, $X + Y = M(w_1)$, $\max\{\|X\|_{\op},\|Y\|_{\op}\} \le 2\epsilon$, and $\max\{\|X\|_{F},\|Y\|_{F}\} \le 2\sqrt{2}\epsilon$. Hence, Lemma~\ref{lem:exp_map} implies that
	\begin{equation}\label{eq:exp_diff}
	\begin{aligned}
	\|\exp(M(w_1))-\exp(M(w_2)) - M(w_1 - w_2)\|_\op &\le& e^{8\epsilon} - 1 - 8\epsilon \text{ and } \\
	\|\exp(M(w_1))-\exp(M(w_2)) - M(w_1 - w_2)\|_F &\le& e^{8\sqrt{2}\epsilon} - 1 - 8\sqrt{2}\epsilon ~.
	\end{aligned}
	\end{equation}
	We can upper bound both displays in ~\eqref{eq:exp_diff} using the following short technical lemma:
	\begin{lem}\label{claim:exp_claim} Let $t \in [0,\log 2]$. Then $e^t - 1 - t \le t^2$.
	\end{lem}
	\begin{proof}
	Let $f(t) = e^t -1 - t$, and $g(t) = t^2$. Then, $f(0) = f'(0) = g(0) = g'(0) = 0$. Moreover, $f''(t) = e^{t}$, and $g''(t) = 2$. Hence, as long as $0 \le t \le \log 2$, $f(t) = \int_{0}^t \int_{0}^u f''(s)dsdu \le \int_{0}^t \int_{0}^u g''(s)dsdu = g(t)$.
	\end{proof}
	Hence for $\epsilon \le \log 2/4\sqrt{2} \le 1/256$, \eqref{eq:exp_diff} and Lemma~\ref{claim:exp_claim} combine to imply
	\begin{eqnarray*}
	\opnorm{\exp(M(w_1))-\exp(M(w_2)) - M(w_1 - w_2)} &\le& 64\epsilon^2 \text{ and } \\
	\fronorm{\exp(M(w_1))-\exp(M(w_2)) - M(w_1 - w_2)} &\le& 128\epsilon^2~.
	\end{eqnarray*}
	Hence, by the triangle inequality, for $\epsilon \le 1/256$
	\begin{eqnarray*}
	\|\exp(M(w_1))-\exp(M(w_2))\|_{\op} &\ge& \|M(w_1 - w_2)\|_{\op} - 64\epsilon^2 \\
	&=& \|w_1 - w_2\| - 64\epsilon^2  \ge \epsilon/2 - 64\epsilon^2 \ge \epsilon/4~,
	\end{eqnarray*}
	and, again, for $\epsilon \le 1/256$
	\begin{eqnarray*}
	\fronorm{\exp(M(w_1))-\exp(M(w_2))} &\le& \|M(w_1 - w_2)\|_{F} + 128\epsilon^2 \\
	&=& \sqrt{2}\|w_1 - w_2\| + 128\epsilon^2 \le 2\sqrt{2}\epsilon + \epsilon/2 \le 4\epsilon~.
	\end{eqnarray*}

\subsubsection{Proof of Lemma~\ref{lem:exp_map}\label{sec:exp_map_proof}}

	Let $M_{i,j}(X,Y)$ denote the homogenous monomial of degree $j$ such consisting of the $\binom{j}{i}$-products of $X$ $i$-times, and $Y$ $j-i$-times. Note then that $M_{j,j}(X,Y) = X^j$, so that $(X+Y)^{j}-X^j = -X^j + \sum_{i=0}^j M_{i,j}(X,Y) = \sum_{i=0}^{j-1} M_{i,j}(X,Y)$. Moreover, by the sub-multiplicativity of $\|\cdot\|$, we have $\|M(X,Y)\| \le \binom{j}{i}\|X\|_2^{i}\|Y\|_2^{j-i} = \binom{j}{i}K^j$.
	\begin{eqnarray*}
	\|\exp(X + Y) - \exp(X) - Y\| &=&  \|(j!)^{-1}\sum_{j=2}^{\infty} (X+Y)^{j}-X^j\|_\op\\
	&=& \|\sum_{j=2}^{\infty}(j!)^{-1}\sum_{i=0}^{j} M_{i,j}(X,Y) \| \\
	&\le& \sum_{j=2}^{\infty} (j!)^{-1}\sum_{i=0}^{j-1} \|M_{i,j}(X,Y) \| \\
	&\le& \sum_{j=2}^{\infty} (j!)^{-1}\sum_{i=0}^{j-1}  \binom{j}{i}K^j\\
	&\le& \sum_{j=2}^{\infty} (j!)^{-1}(2K)^j = e^{2K} - 1 - 2K~.
	\end{eqnarray*}

\subsubsection{Proof of Lemma~\ref{kl_comp}\label{sec:kl_comp}}
For a matrix $M$, let $M_i$ denote the $i$-th row of $M$, and let $M^{\otimes 2} := M^\top M$
\begin{eqnarray*}
\KL(\Pr_{\eig O}^{(T)},\Pr_{A}^{(T)})
&=& \Exp_{\eig O}\left[\sum_{t=1}^T\sum_{i=1}^n\left\langle (\eig O-A)_i, X_t \right\rangle^2 \right]\\
&=& \Exp_{\epsilon_1,\dots,\epsilon^T}\left[\sum_{t=1}^T\sum_{i=1}^n\left\langle (\eig O-A)_i , \sum_{s=1}^t \eig^{t - s}O^{t-s}\epsilon_s \right\rangle^2 \right]\\
&=& \sum_{t=1}^T\sum_{i=1}^n\left\langle (\eig O - A)_i, \Exp_{\epsilon_1,\dots,\epsilon_T}\left[ \left(\sum_{s=1}^t \eig^{t - s}O^{t-s}\epsilon_s\right)^{\otimes 2}\right]\left(\eig O-A\right)_i \right\rangle~.\\
\end{eqnarray*}
We may now compute that, for any $t \in [T]$,
\begin{eqnarray*}
&&\Exp_{\epsilon_1,\dots,\epsilon_T}\left[ \left(\sum_{s=1}^t \eig^{t - s}O^{t-s}\epsilon_s\right)^{\otimes 2}\right] \\
&=& \sum_{s = 1}^t \Exp[ \eig^{2(t-s)} (O^{2(t-s)})^{\top} O^{2(t-s)} \epsilon_s^2] + \sum_{1 = s \ne s' \le t} \Exp[ \eig^{2t-s - s'} (O^{2(t-s)})^{\top} O^{2(t-s')} \epsilon_s \epsilon_{s'}]\\
 &=& \sum_{s = 1}^t \Exp[ \eig^{2(t-s)} (O^{2(t-s)})^{\top} O^{2(t-s)}] = \sum_{s = 1}^t \eig^{2(t-s)}  I~.
\end{eqnarray*}
Hence, we have,
\begin{eqnarray*}
\KL(\Pr_{\eig O}^{(T)},\Pr_{A}^{(T)})  &=& \sum_{t=1}^T\sum_{i=1}^n\left\langle (\eig O-A)_i, \left(\sum_{s=1}^t \eig^{2(t - s)}\right) I  \cdot (\eig O-A )_i \right\rangle\\
&=& \sum_{t=1}^T\left(\sum_{s=1}^t \eig^{2(t - s)}\right) \sum_{i=1}^n \|(\eig O-A)_i\|^2_2  \\
&=& \|\eig O-A\|_F^2 \sum_{t=1}^T (\sum_{s=1}^t \eig^{2(t - s)} = \|\eig O-A\|_F^2\left( \sum_{t=1}^T \sum_{s=0}^{t-1} \eig^{2s}\right) ~.
\end{eqnarray*}

\subsection{Proof of Lemma~\ref{lem:Birge}\label{BirgeSec}}Birge's inequality states that $\sum_{i=1}^N \KL(\Pr_i,\Pr_0) \ge (1-\delta)\log (N\frac{1-\delta}{\delta}) + \delta \log( \frac{\delta}{1 - \delta/N})$ \citep{boucheron13}. Observe that $\delta \log(\frac{\delta}{1 - \delta/N}) \ge \delta \log \frac{\delta}{N(1-\delta)}  = -\delta \log\frac{N(1-\delta)}{\delta}$. Hence $\sum_{i=1}^N \KL(\Pr_i,\Pr_0) \ge (1- 2\delta) \log( \frac{(1-\delta)N}{\delta}) \ge (1- 2\delta) \log \frac{N}{2\delta}$ for $\delta < 1/2$. For the second statement, choose $\calE_{A} := \{\opnorm{A - \ALS(T)} < \epsilon\}$ for $A \in \calN$. Since $\calN$ is $2\epsilon$-separated in $\opnorm{\cdot}$, all $\calE_i$ are disjoint. Hence, for any $A_0 \in \calN$
\begin{eqnarray}
(1-2\delta) \log(|\calN|/2\delta) \le \frac{1}{|\calN|-1}\sum_{A \in \calN - \{A_0\}}^N \KL(\Pr_A^{(T)},\Pr_{A_0}^{(T)}) \le \sup_{A \in \calN - \{A_0\}}^N \KL(\Pr_A^{(T)},\Pr_{A_0}^{(T)}) \:.
\end{eqnarray}
Since $A_0$ was arbitrary, we may pass to an $\inf$ over all $A_0 \in \calN$.

\begin{comment}

For $\eig = 1$, then we have $\sum_{t=1}^T (\sum_{s=1}^t \eig^{2s)} = \sum_{t=1}^T t = T^2/2$. For $\eig \ne 1$,
\begin{eqnarray*}
\sum_{t=1}^T \sum_{s=1}^t\eig^{2t} &=& \sum_{t=1}^T \eig^2\frac{\eig^{2t} - 1}{\eig^2 - 1} \\
&=&\frac{\eig^2}{\eig^2 - 1}\left( -T + \sum_{t=1}^T\eig^{2t} \right) = \frac{\eig^2}{\eig^2 - 1}\left(\eig^2\frac{\eig^{2T} - 1}{\eig^2 - 1} - T  \right)\\
\end{eqnarray*}
\begin{eqnarray*}
\|\exp(X + Y) - \exp(X) - Y\|_\op &=&  \|(j!)^{-1}\sum_{j=2}^{\infty} (X+Y)^{j}-X^j\|_\op\\
&=& \|(j!)^{-1}\sum_{j=2}^{\infty}\sum_{i=0}^{j-1} M_{i,j}(X,Y) \|_\op \\
&\le& \sum_{j=2}^{\infty} (j!)^{-1}\sum_{i=0}^{j-1} \|M_{i,j}(X,Y) \| \\
&\le& \sum_{j=2}^{\infty} (j!)^{-1}\sum_{i=0}^{j-1}  \binom{j}{i}\|X\|_2^{i}\|Y\|_2^{j-i} \\
&=& \|Y\|_2\sum_{j=2}^{\infty}(j!)^{-1}\sum_{i=0}^{j-1} \binom{j}{i}\|X\|_2^{i}\|Y\|_2^{j-i-1} \\
&\le& \|Y\|_2\sum_{j=2}^{\infty}(j!)^{-1}\sum_{i=0}^{j-1} \binom{j}{i} K^{j-1}\\
&\le& \|Y\|_2\sum_{j=2}^{\infty}(j!)^{-1}\frac{1}{K}\sum_{i=0}^{j} \binom{j}{i} K^{j}\\
&=& \|Y\|_2\sum_{j=2}^{\infty}(j!)^{-1}\frac{1}{K}(2K)^j \\
&=& 2\|Y\|_2\sum_{j=2}^{\infty}(j!)^{-1}(2K)^{j-1} = 2\|Y\|_2(\exp(2K)-1)
\end{eqnarray*}
Similarly, one can also just bound $\sum_{i=0}^{j-1} \|M_{i,j}(X,Y) \| \le $
\end{comment}

%!TEX root = LWM.tex
\section{}
Our second lower bound adresses the extent to which it is sharp to control the operator norm error in terms of $n \lambda_{\min}(\sum_{t=1}^T X_tX_t^\top)^{-1}$, instead of the trace-of-the-inverse that typically appears in least squares error bounds. We show that in the regime where $n \gtrsim d$, we can always construct a sequence of Gaussian Responses $Y_t | X_t \sim \mathcal{N}(A_*X_T,\sigma^2)$ such that an \emph{operator norm} error of $n \lambda_{\min}(\sum_{t=1}^T X_tX_t^\top)^{-1}$ is in fact necessary. This is a consequence of the fact that the operator-norm error involves a supremum over all directions in $\calS^{n-1}$. 
\begin{thm}\label{thm:alg_spec_lb} Let $X_1,\dots,X_T \in \R^d$ be an arbitrary dynamical process, let $\epsilon_1,\dots,\epsilon_t \overset{i.i.d.}{\sim} \calN(0,I_n), $and let $Y_t = A_*X_t + \epsilon_t$ for $t = 1,\dots,T$. Then,
\begin{eqnarray*}
\Exp\left[\|\widehat{A}(T) - A_*\|_{\op}^2 \big{|} X_1,\dots,X_T\right] \ge \max\left\{\tr\left(\left(\sum_{t=1}^T X_tX_t\right)^{-1}\right), \sqrt{n}\lambda_{\min}\left(\sum_{t=1}^T X_tX_t\right)^{-1}\right\}
\end{eqnarray*}
\end{thm}
\maxs{This states that at least from our assumptions you can't prove anything}

\subsection{Proof of Theorem~\ref{thm:alg_spec_lb}}
	 Condition on $X_1,\dots,X_T$, we may assume without loss of generality that $\matX$ is deterministic. We let
	$\matX \in \R^{T \times d}$ denote the matrix whose rows are $X_t$, and $\mateps \in \R^{T \times n}$ denote the matrix
	whose rows are $\epsilon_t$. Then, $\widehat{A}(T) - A_* = \matX^{\dagger} \mateps
	$. Moreover, if $v_* \in \arg\max_{v}\|v^{\top}\matX^{\dagger}\|_2$, then $\|v^{\top}\matX^{\dagger}\|_2 = \lambda_{\min}(\sum_{t=1}^T X_tX_t)^{-1/2}$. Moreover, $v_*$ depends only on $\matX$ which by construction is independet of $\mateps$. Hence, $v^{\top}\matX^{\dagger}\mateps \sim \mathcal{N}(\mathbf{O}_n,\lambda_{\min}(\sum_{t=1}^T X_tX_t)^{-1} \cdot I_n) $. And hence,
	\begin{eqnarray*}
	\Exp\left[\|\widehat{A}(T) - A_*\|_{\op}^2\right] &=& \Exp\left[\left\|\matX^{\dagger} \mateps\right\|_{\op}^2\right] = \Exp\left[\sup_{v \in \calS^{d-1}} \left\|v^{\top}\matX^{\dagger} \mateps\right\|_2^2\right] \\
	&\ge& \Exp\left[ \left\|v_*^{\top}\matX^{\dagger} \mateps\right\|_2^2\right] \\
	&=& \Exp\left[\left\|w\right\|^2 \big{|} w \sim \mathcal{N}(\mathbf{O}_n,\lambda_{\min}\left(\sum_{t=1}^T X_tX_t)^{-1/2} \cdot I_n\right)\right] = n\lambda_{\min}\left(\sum_{t=1}^T X_tX_t\right)^{-1}
	\end{eqnarray*}
	On the other hand, we also have 
	\begin{eqnarray*}
	\Exp\left[\left\|\widehat{A}(T) - A_*\right\|_{\op}^2\right] &=& \Exp\left[\left\|\matX^{\dagger} \mateps\right\|_{\op}^2\right] \ge \Exp\left[\left\|\matX^{\dagger} \mateps e_1\right\|_2^2\right] \\
	&=& \Exp\left[ w^{\top}\matX^{\dagger \top}\matX^{\dagger} w \big{|} w \sim \mathcal{N}(\mathbf{O}_T, I_T)\right] \\
	&=& \tr\left(\matX^{\dagger \top}\matX^{\dagger}\right) = \tr\left(\left(\sum_{t=1}^T X_tX_t\right)^{-1}\right)
	\end{eqnarray*}

\end{document}

%%% Local Variables:
%%% mode: latex
%%% TeX-master: t
%%% End: